\documentclass{article}

\usepackage{amsmath, amssymb, amsthm, array, bbm, booktabs, color, enumitem, float, graphicx, mathtools, multicol, natbib, makecell, rotating, todonotes, url}
\usepackage[top=1.0in, bottom=1.2in, left=1.2in, right=1.2in]{geometry}
\usepackage[breaklinks, colorlinks=true, citecolor=black, urlcolor=blue, linkcolor=black]{hyperref}
\usepackage{subcaption}  
\usepackage[dvipsnames]{xcolor}

\newcommand{\real}{\mathbb{R}}

\newcommand{\myQ}{\mathbb{Q}}
\newcommand{\myE}{\mathbb{E}}

\newcommand{\cL}{\mathcal{L}}
\newcommand{\cP}{\mathcal{P}}

\newcommand{\cT}{\mathcal{T}}

\newcommand{\cY}{\mathcal{Y}}

\newcommand{\cm}{$\checkmark$}
\newcommand{\hsp}{\hspace{0.2mm}}
\newcommand{\dd}{\hsp \mathrm{d}}
\newcommand{\myS}{\mathsf{S}}
\newcommand{\myT}{\mathsf{T}}
\newcommand{\one}{\mathbbm{1}}

\newcommand{\NC}{\operatorname{NC}}

\newcommand{\AC}{\operatorname{AC}}
\newcommand{\MC}{\operatorname{MC}}

\newcommand{\CoC}{\operatorname{CoC}}
\newcommand{\CwC}{\operatorname{CwC}}
\newcommand{\CwCoC}{\operatorname{CwCoC}}
\newcommand{\CMC}{\operatorname{ModC}}

\newcommand{\CC}{\operatorname{CC}}
\newcommand{\fCC}{\operatorname{fCC}}
\newcommand{\pCC}{\operatorname{pCC}}

\newcommand{\DC}{\operatorname{DC}}
\newcommand{\EC}{\operatorname{EC}}

\newcommand{\PC}{\operatorname{PC}}
\newcommand{\aPC}{\operatorname{aPC}}
\newcommand{\fPC}{\operatorname{fPC}}
\newcommand{\pPC}{\operatorname{pPC}}
\newcommand{\WPC}{\operatorname{WPC}}

\newcommand{\QC}{\operatorname{QC}}
\newcommand{\SQC}{\operatorname{SQC}}
\newcommand{\UQC}{\operatorname{UQC}}

\newcommand{\TC}{\operatorname{TC}}
\newcommand{\fTC}{\operatorname{fTC}}
\newcommand{\pTC}{\operatorname{pTC}}

\newcommand{\UC}{\operatorname{UC}}

\newcommand{\LS}{\operatorname{Lev}}

\newcommand\SEarrow{\mathrel{\rotatebox[origin=c]{-45}{$\Rightarrow$}}}
\newcommand\hookSEarrow{\mathrel{\rotatebox[origin=c]{-45}{$\hookrightarrow$}}}
\newcommand\SWarrow{\mathrel{\rotatebox[origin=c]{45}{$\Leftarrow$}}}

\newcommand\hookSWarrow{\mathrel{\rotatebox[origin=c]{45}{$\hookleftarrow$}}}

\theoremstyle{plain}  
\newtheorem{theorem}{Theorem}[section] 
 
\newtheorem{proposition}[theorem]{Proposition}
\newtheorem{corollary}[theorem]{Corollary} 

\theoremstyle{definition} 
\newtheorem{definition}[theorem]{Definition}

\newtheorem{example}[theorem]{Example}

\theoremstyle{remark} 
\newtheorem{remark}[theorem]{Remark} 

\setlist[enumerate,1]{label={(\alph*)}}

\title{Hierarchies of Calibration: Classification meets Regression}
\author{Johannes Resin$^{1,3}$, Lu Yang$^{^2}$, and Tilmann Gneiting$^{3,4}$ \\ \\
        $^1$Goethe University Frankfurt \\ $^2$University of Minnesota \\ $^3$Heidelberg Institute for Theoretical Studies \\ $^4$Karlsruhe Institute of Technology}
\date{\today}

\begin{document}

\maketitle

\begin{abstract}
Concepts of calibration formalize the compatibility between probabilistic predictions and the respective outcomes.  In a nutshell, the outcomes ought to be indistinguishable from random draws from the predictive distributions.  In this paper, we review, extend, and bridge notions of calibration that have been proposed for classification and regression tasks.  Particular emphasis is given to hierarchical relations between the various notions, as they apply to general real-valued data, continuous outcomes, count data, nominal classes, and binary outcomes.  To highlight a number of contributions, we introduce the notion of modal calibration for nominal outcomes, we distinguish full, partial, and average calibration in this setting, and we show that double probability integral transform (PIT) calibration is logically independent of previously proposed concepts of calibration for discrete outcomes.  Furthermore, we generalize extant results on concepts of calibration that are expressed in terms of properties or functionals of the predictive distributions, such as means, quantiles, or event probabilities.  Throughout the paper, we illustrate the concepts and their hierarchical relations in worked examples, and we provide algorithmic tools that support the construction of instructive examples and counterexamples. 

\bigskip

\end{abstract}

\tableofcontents

\section{Introduction}  \label{sec:introduction}

There is a broad consensus in the scientific community that predictions ought to reflect inherent uncertainty and, therefore, ought to be issued in the form of probability distributions over possible outcomes.  The concept of calibration is closely related to the general principle of outcome indistinguishability \citep{Dwork2021, Gopalan2025}, as expressed by \citet[p.~129]{Gneiting2014}: 
\begin{quote}  \small
Calibration concerns the statistical compatibility between the probabilistic forecasts and the realizations; essentially, the observations should be indistinguishable from random draws from the predictive distributions.
\end{quote}
Early studies of calibration from theoretical and methodological perspectives include the work of \citet{Dawid1984}, \citet{Diebold1998}, \citet{Zadrozny2002}, \citet{Gneiting2007a}, and \citet{Gneiting2013}, whereas \citet{VanCalster2016}, \citet{Alba2017}, and \citet{Nolde2017} motivate calibration from applied perspectives.  Recent large-scale, empirical studies in classification \citep{Guo2017, Minderer2021} and regression \citep{Dheur2023} have brought calibration to the attention of a broad audience in machine learning, prompting vigorous further development \citep{Kuleshov2018, Kull2019, Kumar2019, Song2019, Vaicenavicius2019, Gupta2022}.  Expositions of recent progress in this area are included in the papers by \citet{Gneiting2023a}, \citet{Marx2023}, \citet{Silva2023}, and \citet{Derr2025}, among others.  

The case of count data, where the outcome is a natural number, can be considered a middle ground between classification and regression, and there has been a small, specialized literature that studies notions of calibration specifically for count data \citep{Czado2009, Kolassa2016}.  In regression settings with continuous outcomes, probabilistic calibration is a popular concept based on checks of the uniformity of the probability integral transform (PIT).  In discrete settings, randomization is required to achieve uniformly distributed PIT values for ideal forecasts.  Various authors have studied probabilistic calibration in such settings \citep{Dunn1996, Stasinopoulos2017, Liu2018} and suggested variants of the PIT that avoid randomization \citep{Czado2009, Li2012, Shepherd2016, Liu2025}.  A related, but different approach for count data was proposed by \citet{Yang2024}, who designed the concept of double PIT calibration.  

The strongest form of calibration is auto-calibration, which posits that the conditional distribution of the outcome, given the forecast, equals the forecast itself.  However, with the exception of binary outcomes, where the forecast is simply a predictive probability, auto-calibration is difficult to check in practice.  Therefore, the assessment of calibration is based on weaker, more tractable notions of calibration.  Evidently, the definitions in this paper represent idealized goals that in typical practice can only be satisfied in an approximate sense.  While the treatment in this paper is purely theoretical, ignoring many practical issues such as estimation uncertainty in both underlying models and diagnostic procedures, the clear distinction between the various notions of calibration provides practical insights that elucidate their connections, uncover shortcomings and blind spots, and inform methodological development.  For instance, calibration error metrics \citep[e.g.,][]{Gupta2022} and related score decompositions \citep[e.g.,][]{Silva2023} typically rely on weaker notions of calibration that are more easily assessed in practice.  The choice of such a notion directly influences calibration error estimates and other score components, as illustrated by \citet{Arnold2024} and \citet{Resin2026}.  Thus, a proper understanding of these concepts and their relationships is vital in practice, to inform the development of tools for calibration checks, and to avoid labeling potentially miscalibrated predictions as calibrated.  Suitably tailored combinations of calibration checks based on distinct notions of calibration make assessments more robust to pitfalls, and provide guidance in the (re)calibration of predictions at hand.  

To delineate our work from the existing literature, we focus on the development of rigorous hierarchies of calibration notions, which are summarized in Figures \ref{fig:hierarchy_classification}--\ref{fig:hierarchy_T}.  We note a general trade-off, in that stronger notions of calibration are more difficult to assess in practice, and we observe a partial collapse of the hierarchies as the outcome space narrows.  In the simplest case of a binary outcome, many otherwise genuinely weaker notions coincide with auto-calibration.  In more complex settings, care needs to be taken to avoid reliance on techniques that make forecasts appear calibrated with respect to a particular notion of calibration, even though they might be far from the goal of auto-calibration.

The paper is organized as follows.  Section \ref{sec:notions} starts from extant notions of calibration, as developed in two separate strands of literature, namely, for classification tasks and for regression problems, respectively.  In classification tasks, the outcome is nominal.  For this setting, we introduce the thus far unnoticed concept of modal calibration and provide rigorous definitions of extant notions that allow for ties in the quoted probabilities.  In regression problems, the outcome is real-valued and ordered.  A plethora of notions of calibration for either setting exists, and we study their hierarchical relations in detail.  In Section \ref{sec:middle}, we aim to bridge, extend, and unify the two strands of literature.  We study the recently developed notion of double PIT calibration for count data, extend it to the general case of real-valued outcomes, and show that it is logically independent of previously proposed concepts of calibration.  We adapt notions proposed in regression settings so that they fit classification settings, by distinguishing full, partial, and average calibration.  Furthermore, we generalize extant results on concepts of calibration that are expressed in terms of properties or functionals of the predictive distributions, such as means, quantiles, or threshold (non) exceedance probabilities.  The main body of the paper closes with a discussion in Section \ref{sec:discussion}.  Many proofs are deferred to Appendix \ref{app:proofs}.

Examples and counterexamples assume a prominent role in our work.  Throughout the paper, we highlight the calibration concepts and their hierarchical relations in worked analytic examples, which are summarized in Tables \ref{tab:examples_classification}--\ref{tab:examples_classification_extended}.  We elucidate technical details for examples both from the extant literature and in this paper in Appendix \ref{app:examples}.  In supplementary material, which uses methods from linear algebra and is available online at \url{https://github.com/resinj/replication_hierarchies}, we provide algorithmic tools that support the construction of examples and counterexamples with particular calibration properties.

\section{Notions of calibration}  \label{sec:notions}

We work in the standard prediction space setting \citep{Gneiting2013}, where we consider the joint law $\myQ$ of a (random) distributional forecast $F$ and a target variable $Y$.  All (in)equalities are understood to hold in an almost sure sense with respect to the probability measure $\myQ$.  The outcome $Y$ takes values in a set $\cY$, which is identified with a subset of the real line, $\real$.  This identification can be arbitrary, particularly in the case of a nominal target variable.  We adopt the custom of identifying the forecast $F$ with its cumulative distribution function (cdf).  The corresponding (random) Borel probability measure is denoted by $P_F$,
and $f$ denotes any corresponding Lebesgue density or probability mass function, respectively. We call a random variable or probability measure continuous if the corresponding cdf is continuous. Occasionally, the symbol $\cL(X \mid C)$ is used to denote the conditional law of a random variable $X$ given some conditioning variable $C$.

\subsection{Auto-calibration and marginal calibration}  \label{sec:AC_MC}

While many notions of calibration apply to a particular setting --- that is, a particular restriction of the outcome space $\cY$ --- the notions of auto-calibration \citep{Tsyplakov2011} and marginal calibration \citep{Gneiting2007a} apply universally.

\begin{definition}  \label{def:AC_MC}
The forecast $F$ for the outcome $Y$ satisfies
\begin{enumerate}
\item  \emph{auto-calibration}\/ ($\AC$) if
\begin{align}  \label{eq:AC}
\myQ \left( Y \in B \mid F \right) = P_F(B) \quad \text{for all Borel sets} \quad B \subseteq \real;
\end{align}
\item  \emph{marginal calibration}\/ ($\MC$) if 
\begin{align}  \label{eq:MC}
\myQ \left( Y \in B \right) = \myE_{\hsp \myQ} \left[ P_F(B) \right] \quad \text{for all Borel sets} \quad B \subseteq \real.
\end{align} 
\end{enumerate}
\end{definition}

Informally, a forecast is ideal if it equals the conditional distribution of the outcome given the information at hand.  \citet{Gneiting2013} formalize this notion, and it follows that a forecast is auto-calibrated if, and only if, it is ideal relative to the information set or $\sigma$-algebra generated by itself.  Auto-calibration formalizes the general principle of outcome indistinguishability, ensuring that observed outcomes are entirely indistinguishable from random numbers drawn from the forecast distributions \citep{Gneiting2014, Dwork2021}.

Auto-calibration is the strongest notion of calibration that adheres to the principle advocated by \citet{Dawid1984}, namely, that the notion can be formalized in terms of the forecast, $F$, and the target or outcome, $Y$, only, with the exception of allowing for randomization devices.  While the concept of auto-calibration can be strengthened by conditioning on covariates \citep[e.g.,][]{VanCalster2016}, as recently studied under the heading of multicalibration and relating to issues of fairness \citep{HebertJohnson2018, Noarov2023}, we restrict the discussion in this paper to notions of calibration that honor Dawid's principle.  

Marginal calibration emerges from the auto-calibration condition by integrating it with respect to the marginal distribution of the forecast $F$.  Therefore, auto-calibration implies marginal calibration.  
The marginal calibration condition \eqref{eq:MC} can be written equivalently in terms of thresholds, in that 
\begin{align}  \label{eq:MC_alternative}
\myQ \left( Y \leq y \right) = \myE_{\hsp \myQ} \left[ \hsp F(y) \right] \quad \text{for all} \quad y \in \real.  
\end{align}
If the outcome space $\cY$ is countable and the forecast $F$ has probability mass function $f$, the auto-calibration condition \eqref{eq:AC} is equivalent to 
\begin{align}  \label{eq:AC_countable}
\myQ \left( Y = y \mid f \right) = f(y) \quad \text{for all} \quad y \in \cY, 
\end{align}
and the marginal calibration condition \eqref{eq:MC} is equivalent to 
\begin{align}  \label{eq:MC_countable}
\myQ \left( Y = y \right) = \myE_{\hsp \myQ} \left[ \hsp f(y) \right] \quad \text{for all} \quad y \in \cY.
\end{align}

The following example abstracts and subsumes a great number of sampling-based approaches to the generation of predictive distributions, including but not limited to forecast ensembles \citep{Gneiting2005a}, Markov chain Monte Carlo techniques \citep{Kruger2021b}, and modern generative neural networks \citep{Goodfellow2020}.  Its practical relevance can hardly be overstated.

\begin{example}[simple random sample]  \label{ex:SRS}
Consider a fixed distribution $G$ with strictly positive variance, and let $F_n$ be the empirical cdf of a simple random sample $X_1, \ldots, X_n$ of size $n$ from $G$, that is, 
\begin{align}  \label{eq:SRS}
F_n(y) = \frac{1}{n} \sum_{i=1}^n \one \left( y \geq X_i \right) \quad \text{for} \quad y \in \real. 
\end{align}
Furthermore, let the outcome $Y$ have distribution $G$ and be independent of $X_1, \ldots, X_n$.  Then the random forecast $F_n$ for $Y$ is marginally calibrated, but fails to be auto-calibrated.  We return to this example in Tables \ref{tab:examples_classification} and \ref{tab:examples_classification_extended}, where we assume that $G$ is multinomial, and also in Table \ref{tab:examples_regression}, where we assume that $G$ is continuous.
\end{example}

In the next two sections, we consider two widely studied settings of particular interest with established notions of calibration in the extant literature, namely, classification and regression.

\subsection{Classification: A hierarchy of calibration for nominal targets}  \label{sec:classification}

Classification tasks concern nominal targets or classes that lack a natural ordering.  Without loss of generality, we assume that there is a total of $k$ classes labeled $1, \ldots, k$, respectively, thus identifying the outcome space $\cY$ with the set $\{ 1, \ldots, k \}$.  Probabilistic statements are typically expressed in terms of the probability mass function $f$ on the set $\cY$.  The natural point prediction in classification tasks is the modal set of $f$, for which we write 
\begin{align}\label{eq:mode}
\hat{y}_F = \arg \max_{y \in \{ 1, \ldots, k \}} f(y) = \{ y \in \{ 1, \ldots, k \} : f(y) = m_F \}, 
\end{align}
where 
\begin{align*}
m_F = \max \{ f(1), \ldots, f(k) \}
\end{align*}
is the modal probability.  The probability mass of the modal set then is $| \hsp \hat{y}_F| \, m_F = P_F(\hat{y}_F)$.  The modal set is a possibly set-valued functional, as its cardinality $| \hsp \hat{y}_F|$ may exceed one when there are ties in the predicted class probabilities.  While ties are uncommon with neural networks in typical machine learning tasks, they may arise in many other settings, including but not limited to ensemble forecasts and subjective expert predictions, as well as trees and forests trained on small (sub)samples.

Formal definitions of notions of calibration for nominal targets have typically ignored ties.  In view of the applied relevance of the aforementioned settings, we provide definitions of popular notions that allow for ties.\footnote{The definition of confidence calibration in the possible presence of ties that is presented in Definition~\ref{def:classification}~(b) is the result of much discussion (and some initial confusion) among the authors.  We settled on \eqref{eq:CoC} as this aligns well with the stronger requirement \eqref{eq:CwCoC} and, importantly, condition \eqref{eq:CoC} can be checked in a single reliability diagram, in line with the original spirit of confidence calibration \citep{Guo2017}.}  Furthermore, we introduce the complementary notion of modal calibration.  All conditions are understood in the almost sure sense with respect to the joint distribution $\myQ$.

\begin{definition}  \label{def:classification}
The forecast $F$ for a nominal outcome $Y$ satisfies
\begin{enumerate}
\item  \emph{class-wise calibration}\/ ($\CwC$) if
\begin{align}  \label{eq:CwC}
\myQ \left( Y = y \mid f(y) \right) = f(y) \quad \text{for all} \quad y \in \{ 1, \ldots, k \};
\end{align}        
\item  \emph{confidence calibration}\/ ($\CoC$) if
\begin{align}  \label{eq:CoC}
\myQ \left( Y \in \hat{y}_F \mid | \hsp \hat{y}_F| \, m_F  \hsp \right) = | \hsp \hat{y}_F| \, m_F;
\end{align}   
\item  \emph{class-wise confidence calibration}\/ ($\CwCoC$) if
\begin{align}  \label{eq:CwCoC}
\myQ \left( Y \in \hat{y}_F \mid \hat{y}_F, m_F \right) = | \hsp \hat{y}_F| \, m_F;
\end{align} 
\item  \emph{modal calibration}\/ ($\CMC$) if the modal set of the conditional distribution $\cL(Y \mid \hat{y}_F)$ equals\/ $\hat{y}_F$ almost surely, that is, 
\begin{align}\label{eq:CMC}
\arg \max_{y \in \{ 1, \ldots, k \} } \myQ(Y = y \mid \hat{y}_F) = \hat{y}_F.
\end{align}
\end{enumerate}
\end{definition}

To preview subsequent results, class-wise calibration and modal calibration enjoy additional theoretical backing from their interpretation as conditional calibration with respect to distributional properties, for which we refer to Definition \ref{def:T} and Table \ref{tab:examples_T} in Section \ref{sec:T}.  If the mode is single-valued, that is, $| \hsp \hat{y}_F| = 1$ almost surely, the definitions for confidence calibration and class-wise confidence calibration reduce to extant definitions in the literature \citep{Guo2017, Vaicenavicius2019, Gupta2022, Marx2023, Silva2023}.  In this simplified case, the notions in Definitions \ref{def:AC_MC} and \ref{def:classification} have been employed under a perplexing variety of names.  For instance, \citet{Guo2017} introduced confidence calibration as perfect calibration, and \citet{Gupta2022} introduced class-wise confidence calibration as top-label calibration.  Table 2 of \citet{Marx2023} refers to auto-calibration as canonical calibration, to confidence calibration as top-label calibration, and to class-wise calibration as marginal calibration.  The terminology in our definitions seems plausible and might get adopted broadly. 

The following simple example illustrates that the popular concept of confidence calibration and the novel concept of modal calibration capture complementary aspects of calibration.  While confidence calibration assesses the correct specification of the probability mass assigned to the (predicted) modal set only, modal calibration concerns the correct specification of the modal set itself.

\begin{example}  \label{ex:CMC}
Consider a deterministic forecast $F$ with probability mass $f(\hsp j)$ for $j = 1, \ldots, k$, and let $g(\hsp j) = \myQ( Y = j)$ denote the corresponding outcome probability. 
\begin{enumerate}
\item 
Let $k = 2$, and let $F$ be given by $f(1) = 1$ and $f(2) = 0$.  Let $g(1) \in (1/2,1)$ and $g(2) = 1 - g(1) < 1/2$.  Then the forecast $F$ is modally calibrated, but neither confidence calibrated nor marginally calibrated.
\item  
Let $k = 3$, and let the forecast $F$ be given by $f(1) = 4/10$ and $f(2) = f(3) = 3/10$.  Let $g(1) = 4/10$, $g(2) = 5/10$, and $g(3) = 1/10$.  Then $F$ is confidence calibrated and class-wise confidence calibrated, but neither modally calibrated nor marginally calibrated.  
\end{enumerate}
\end{example}

Let us now allude to practical considerations in the assessment of calibration, where we recall that a reliability diagram is a plot of the conditional probability of a certain outcome against the respective forecast probability, for which we recommend the stable version of \citet{Dimitriadis2021}.  To check for confidence calibration, a single reliability diagram suffices, whereas checks for class-wise confidence calibration require $k$ or (in the case of ties) even more reliability diagrams.  To check modal calibration, if the mode is unique almost surely, it suffices to compute a $k \times k$ confusion matrix.  Therefore, if a balance between computational and cognitive efforts, and strength of the check is sought, a possible pragmatic approach would be to check for modal calibration in concert with marginal calibration and confidence calibration.  Such an approach seems to serve a vast majority of applied settings, and relates to the ubiquitous practice of evaluating hard classifiers via confusion matrices \citep{Sokolova2009}, but remains to be implemented and explored in practice.  Additional calibration diagnostics may check extensions of confidence calibration that assess the confidence in the top-$\ell$ class predictions \citep[Eq.\ (6)]{Gupta2022} via $\ell \geq 2$ reliability diagrams.

\begin{figure}[t]
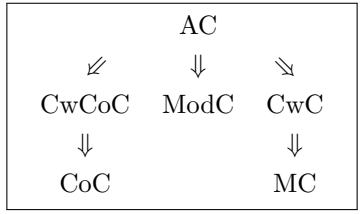
  
\centering
\fbox{
\begin{tabular}{ccc}
& $\AC$ & \\ [1mm]
\hspace{2mm} $\SWarrow$ & $\Downarrow$ & \hspace{-4mm} $\SEarrow$ \\ [1mm]
$\CwCoC$ & $\CMC$ & $\CwC$ \\ [1mm]
$\Downarrow$ & & $\Downarrow$ \\ [1mm]
$\CoC$ & & $\MC$ \\
\end{tabular}
}
\caption{Hierarchy of calibration for nominal outcomes in terms of the notions and acronyms introduced in Definitions \ref{def:AC_MC} and \ref{def:classification}.  \label{fig:hierarchy_classification}}
\end{figure}

The implications in Figure \ref{fig:hierarchy_classification} are straightforward to demonstrate and do not require worked proofs.  For ease of reference, we collect the implications in the following proposition. 

\begin{proposition}  \label{prop:implications_classification}
For a nominal outcome, auto-calibration implies class-wise calibration, class-wise confidence calibration, and modal calibration; class-wise calibration implies marginal calibration; and class-wise confidence calibration implies confidence calibration.
\end{proposition}

\begin{table}[t]
\caption{Calibration properties for nominal outcomes in terms of the notions in Definitions \ref{def:AC_MC} and \ref{def:classification} for examples from the extant literature (upper block) and in this paper (lower block).  \label{tab:examples_classification}}
\centering
\footnotesize
\begin{tabular}{llccccccc}
\toprule
Source                           & Specifics                    & $\AC$ & $\CwCoC$ & $\CoC$ & $\CMC$ & $\CwC$ & $\MC$ \\
\midrule
\citet{Vaicenavicius2019}        & Table 1                      & --    & \cm      & \cm    & \cm & \cm   & \cm \\
\citet{Vaicenavicius2019}        & Supplement, Table 2          & --    & --       & \cm    & \cm & --    & \cm \\
\citet{Silva2023}                & Footnote 2                   & --    & --       & --     & --  & \cm   & \cm \\
\midrule
Section \ref{sec:AC_MC} & Example \ref{ex:SRS}, $G$ multinomial & --    & --       & --     & --  & --    & \cm \\
Section \ref{sec:classification} & Example \ref{ex:CMC} (a)     & --    & --       & --     & \cm & --    & -- \\
Section \ref{sec:classification} & Example \ref{ex:CMC} (b)     & --    & \cm      & \cm    & --  & --    & -- \\
\bottomrule
\end{tabular}
\end{table}

No other implications hold, as the collection of examples in Table \ref{tab:examples_classification} shows.  Therefore, the hierarchy in Figure \ref{fig:hierarchy_classification} is complete, and it reflects the aforementioned practical considerations.  Auto-calibration is the strongest notion, but very hard to check, except for binary settings.  The notions in the middle row are strong, but require at least $k$ reliability diagrams to be checked, except for modal calibration, which can be checked by computing a confusion matrix.  The notions in the bottom row are easier to assess, but weaker than the class-wise versions in the middle row.

\subsection{Regression: A hierarchy of calibration for real-valued outcomes}  \label{sec:regression}

We turn to notions of calibration for real-valued outcomes and introduce additional notation in the prediction space setting.  For a cumulative distribution function (cdf) $F$, let 
\begin{align}  \label{eq:F_inverse}
F^{-1}(\alpha) = \inf \{ y \in \real : F(y) \geq \alpha \} \quad \text{and} \quad F_+^{-1}(\alpha) = \sup\{y \in \real : F(y) \leq \alpha\}
\end{align}
denote the \emph{lower} and \emph{upper quantile} of $F$ at level $\alpha \in (0,1)$, respectively.  We note that $F^{-1}(\alpha) \leq F_+^{-1}(\alpha)$ with equality when $F$ is strictly increasing.  The (randomized) \emph{probability integral transform} (PIT) of the random forecast $F$ is defined as 
\begin{align}  \label{eq:PIT}
Z_F = F(Y-) + U \left( F(Y) - F(Y-) \right),
\end{align}
where $F(y-) = \lim_{x \uparrow y} F(x)$, whereas the random variable $U$ is standard uniform and independent of $F$ and $Y$.  Evidently, if $F$ is continuous, no randomization occurs and the PIT reduces to the simplified form $Z_F = F(Y)$.

\begin{definition}  \label{def:regression}
The forecast $F$ for a real-valued outcome $Y$ satisfies
\begin{enumerate}
\item  \emph{probabilistic calibration}\/ ($\PC$) if $Z_F$ is uniformly distributed on the unit interval;
\item  \emph{conditional (non) exceedance probability calibration}\/ ($\CC$) if 
\begin{align}  \label{eq:CC}
\myQ \left(Z_F \leq \alpha \mid F^{-1}(\alpha) \right) = \alpha \quad \text{for all} \quad \alpha \in (0,1);
\end{align}
\item  \emph{threshold calibration}\/ ($\TC$) if 
\begin{align}  \label{eq:TC}
\myQ \left( Y \leq t \mid F(t) \right) = F(t) \quad \text{for all} \quad t \in \real;
\end{align}
\item  \emph{quantile calibration}\/ ($\QC$) if 
\begin{align}  \label{eq:QC}
\myQ \left( Y < F^{-1}(\alpha) \mid F^{-1}(\alpha) \right) \hsp \leq \hsp \alpha \hsp \leq \hsp 
\myQ \left( Y \leq F^{-1}(\alpha) \mid F^{-1}(\alpha) \right) 
\quad \text{for all} \quad \alpha \in (0,1);
\end{align}
\item  \emph{unconditional quantile calibration}\/ ($\UQC$) if 
\begin{align}  \label{eq:UQC}
\myQ \left( Y < F^{-1}(\alpha) \right) \hsp \leq \hsp \alpha \hsp \leq \hsp \myQ \left( Y \leq F^{-1}(\alpha) \right) 
\quad \text{for all} \quad \alpha \in (0,1);
\end{align}
\item  \emph{weak probabilistic calibration}\/ ($\WPC$) if
\begin{align}  \label{eq:WPC}
\myQ \left( F(Y) < \alpha \right) \hsp \leq \hsp \alpha \hsp \leq \hsp \myQ \left( F(Y-) \leq \alpha \right) 
\quad \text{for all} \quad \alpha \in (0,1).
\end{align}
\end{enumerate}
\end{definition}

The concept of probabilistic calibration has been proposed and studied by \citet{Dawid1984}, \citet{Diebold1998}, and \citet{Gneiting2007a}, among others.  The notion is of fundamental relevance and widely used; it ensures that all prediction intervals bounded by quantiles have the desired, nominal coverage.  Conditional (non) exceedance probability calibration is a conditional version of probabilistic calibration that can be traced to \citet{Mason2007}.  Evidently, conditional (non) exceedance probability calibration implies probabilistic calibration.  Threshold calibration, introduced by \citet{Henzi2021}, uses a conditional version of the marginal calibration condition \eqref{eq:MC_alternative}.  

Quantile calibration has been studied by \citet{Gneiting2023a}, \citet{Arnold2024} and \citet{Allen2025a}.  However, \citet{Gneiting2023a} use an alternative definition of quantile calibration, namely,
\begin{align}  \label{eq:SQC}
q_\alpha \! \left( \cL(Y \mid q_\alpha(F)) \right) = q_\alpha(F) \quad \text{for all} \quad \alpha \in (0,1),
\end{align}
where 
\begin{align}  \label{eq:q_alpha}
q_\alpha(F) = \left[ F^{-1}(\alpha), \hsp F_+^{-1}(\alpha) \right]
\end{align}
denotes the set-valued $\alpha$-quantile, which spans the interval of possible $\alpha$-quantiles.  Due to the limited practical relevance of the latter  notion, we adopt the simpler definition from \cite{Arnold2024} in this paper.  By Proposition 2.18 in \citet{Gneiting2023a}, the strong version of quantile calibration ($\SQC$) in \eqref{eq:SQC} implies quantile calibration in the form \eqref{eq:QC} used here.  However, the reverse is not true as illustrated by Example 2.14 (a) in \citet{Gneiting2023a} and Example 5.2 in \citet{Arnold2025}.  Thus, the requirement \eqref{eq:SQC} is strictly stronger than \eqref{eq:QC}.

Regarding unconditional quantile calibration, it is readily seen that the original condition from \citet{Gneiting2023a}, which uses the set-valued quantile $q_\alpha(F)$ and requires $\myQ(Y < F^{-1}_+(\alpha)) \leq \alpha \leq \myQ(Y \leq F^{-1}(\alpha))$, coincides with the requirement \eqref{eq:UQC}.
The recently proposed notion of weak probabilistic calibration \citep{Allen2025a} is equivalent to unconditional quantile calibration, as noted in Proposition \ref{prop:implications_regression} (d) below.

\begin{figure}[t]
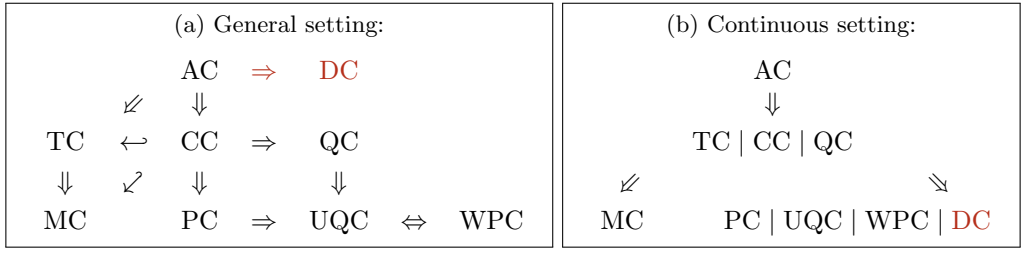
  
\centering
\fbox{\subcaptionbox{General setting:}{
\begin{tabular}{ccccccc}
& & $\AC$ & \color{BrickRed} $\Rightarrow$ & \color{BrickRed} $\DC$ & & \\
& $\SWarrow$ & $\Downarrow$ & & & & \\ [1mm]
$\TC$ & $\hookleftarrow$ & $\CC$ & $\Rightarrow$ & $\QC$ & & \\ [1mm]
$\Downarrow$ & $\hookSWarrow$ & $\Downarrow$ & & $\Downarrow$ & & \\ [1mm]
$\MC$ & & $\PC$ & $\Rightarrow$ & $\UQC$ & $\Leftrightarrow$ & $\WPC$ \\
\end{tabular}
}}
\fbox{\subcaptionbox{Continuous setting:}{
\begin{tabular}{ccc}
& $\AC$ & \\
& $\Downarrow$ & \\ [1mm]
& $\TC \mid \CC \mid \QC$ & \\ [1mm]
$\SWarrow$ & & $\SEarrow$ \\ [1mm]
$\MC$ \rule{1mm}{0mm} & \multicolumn{2}{c}{\rule{3mm}{0mm} $\PC \mid \UQC \mid \WPC \mid \color{BrickRed} \DC$} \\
\end{tabular}
}}
\caption{Hierarchy of calibration for real-valued outcomes, updated from \citet[Figure 1 (b)]{Gneiting2023a}, in terms of the notions and acronyms introduced in Definitions \ref{def:AC_MC} and \ref{def:regression} in (a) the general setting and (b) subject to Assumption 2.15 of \citet{Gneiting2023a}.  Hook arrows indicate conjectured implications.  The connection to double PIT calibration (\textcolor{BrickRed}{$\DC$}) in \textcolor{BrickRed}{red} color is explored in Section \ref{sec:DC}.  \label{fig:hierarchy_regression}}
\end{figure}

The hierarchies in Figure \ref{fig:hierarchy_regression} expand and refine the hierarchies in Figure 1 (b) of \citet{Gneiting2023a} and reflect practical considerations.  Auto-calibration is the strongest notion and implies all the other notions of calibration considered.  However, in regression settings auto-calibration is very difficult, if not impossible, to check.\footnote{For this reason, \citet{Arnold2025} introduce isotonic calibration.  Informally, isotonic calibration weakens auto-calibration by allowing for deviations from the auto-calibration criterion \eqref{eq:AC} in cases where the true conditional distributions show counterintuitive behavior.  By Proposition 5.3 of \citet{Arnold2025}, auto-calibration implies isotonic calibration, whereas isotonic calibration implies both threshold calibration and quantile calibration.  We refrain from further discussion, to avoid the introduction of the only recently developed measure theoretic tools \citep{Arnold2025} that underlie isotonic calibration.  In contrast to the hierarchies in Figure 1 (b) of \citet{Gneiting2023a}, we also omit strong threshold calibration (STC) as it only applies to continuous outcomes.}  The notions in the middle rows of the hierarchies are strong, but require numerous reliability diagrams to be checked, as described by \citet{Gneiting2023a}.  The notions in the bottom row are easy to check but weak.  

Even weaker notions have been used in practice.  For instance, the empirical cdf forecast $F_n$ from Example \ref{ex:SRS} fails to be unconditionally quantile calibrated for a continuous outcome $Y$.  However, it satisfies
\begin{align*}
\myQ \left( F_n(Y) < \frac{i}{n} \right) = \frac{i}{n+1} < \frac{i}{n} < \frac{i+1}{n+1} = \myQ \left( F_n(Y) \leq \frac{i}{n} \right) \quad \text{for} \quad i = 1, \ldots, n, 
\end{align*}
owing to the fact that under simple random sampling the rank of $Y$ when pooled with $X_1, \ldots, X_n$ is uniformly distributed on the integers $1, \ldots, n + 1$.  For ensemble forecasts in the atmospheric sciences, the uniformity of this rank is commonly interpreted as a variant of probabilistic calibration and checked empirically in rank histograms.  For further discussion, see \citet{Hamill2001}, \citet{Gneiting2007a}, and \citet{Vogel2018}, among others.

The following result captures the implications in Figure \ref{fig:hierarchy_regression} that are not yet available from \citet{Gneiting2023a} and the references therein.  For the proof, we refer to Appendix \ref{app:proofs}.

\begin{proposition}  \label{prop:implications_regression}
With respect to the notions in Definition \ref{def:regression},  
\begin{enumerate}
\item  conditional (non) exceedance probability calibration implies quantile calibration;
\item  quantile calibration implies unconditional quantile calibration;
\item  probabilistic calibration implies unconditional quantile calibration; and 
\item  unconditional quantile calibration is equivalent to weak probabilistic calibration. 
\end{enumerate}
\end{proposition}

\begin{table}[t]
\caption{Calibration properties for real-valued outcomes in terms of the notions in Definitions \ref{def:AC_MC} and \ref{def:regression} for examples from the extant literature (upper block) and in this paper (lower block).  \label{tab:examples_regression}}
\centering 
\footnotesize
\begin{tabular}{llccccccc} 
\toprule
Source                       & Specifics                       & $\AC$ & $\CC$ & $\QC$ & $\TC$ & $\PC$ & $\MC$ & $\UQC$ \\ 
\toprule
\citet{Gneiting2023a}        & Figure 3                        & --    & \cm   & \cm   & \cm   & \cm   & \cm   & \cm \\   
\citet{Gneiting2023a}        & Example 2.4 (a)                 & --    & --    & -- \   & --    & \cm   & \cm   & \cm \\   
\citet{Gneiting2023a}        & Example 2.2, unfocused          & --    & --    & --    & --    & \cm   & --    & \cm \\   
\midrule
Section \ref{sec:AC_MC} & Example \ref{ex:SRS}, $G$ continuous & --    & --    & --    & --    & --    & \cm   & -- \\
Section \ref{sec:regression} & Example \ref{ex:dPC}            & --    & --    & \cm   & --    & --    & --    & \cm \\
Section \ref{sec:regression} & Example \ref{ex:TC+nUQC}        & --    & --    & --    & \cm   & --    & \cm   & -- \\
\bottomrule
\end{tabular}
\end{table}

The display of implications in Figure \ref{fig:hierarchy_regression} (a) is complete up to conjectured implications, as the examples in Table \ref{tab:examples_regression} illustrate.  While the notions of calibration considered in the figure and table have been developed in regression settings, they apply to classification settings as well, provided the classes are identified by numbers on the real line, $\real$. In the interest of simplicity and clarity, most of the subsequent (counter)examples consider outcomes that take values in $\cY = \{ 1, \ldots, k \} \subseteq \real$ only, thus admitting a classification interpretation. The forecast in the following example is quantile calibrated, but neither marginally calibrated nor probabilistically calibrated.  Despite its simplicity, we invite the reader to study the example in some detail, as it serves to introduce notation and conventions used throughout the paper.

\begin{example}  \label{ex:dPC}
Let the joint distribution $\myQ$ of the forecast $F$ and the outcome $Y \in \cY = \{ 1, 2, 3 \}$ be described by the following table. 

\begin{center}
\begin{tabular}{ >{$}c<{$} | >{$}c<{$} | >{$}c<{$} >{$}c<{$} >{$}c<{$} | >{$}c<{$} >{$}c<{$} >{$}c<{$} }
\toprule
j & \myQ(F = F_j) & f_j(1) & f_j(2) & f_j(3) & g_j(1) & g_j(2) & g_j(3) \\
\midrule
1 & 1/3           & 1/2    & 1/4    & 1/4    & 3/5    & 7/20   & 1/20 \\ 
2 & 1/3           & 1/4    & 1/2    & 1/4    & 7/20   &  2/5   & 1/4 \\ 
3 & 1/3           & 1/4    & 1/4    & 1/2    & 3/20   & 3/20   & 7/10 \\ 
\bottomrule
\end{tabular}
\end{center}

\noindent  In words, the random forecast $F$ attains three distinct values, $F_1, F_2,$ and $F_3$, with equal probability of $1/3$.  For forecast $j \in \{ 1, 2, 3 \}$ and class label $y \in \cY = \{ 1, 2, 3 \}$, the probability mass function of $F_j$ has value $f_j(y)$, and conditional on $F = F_j$ the target $Y$ has probability mass $g_j(y)$ at $y$.  For a slightly more elaborate interpretation, let us introduce a feature or covariate $X$ that is uniformly distributed on $\{ 1, 2, 3 \}$.  Then we can interpret $f_j(y)$ as the forecaster's assessment that $Y = y$ given $X = j$, and $g_j(y)$ as the true conditional probability that $Y = y$ given $X = j$.  By associating the symbol $f$ with the term $f$\/orecast, and the symbol $g$ with the term $g \hsp$enuine conditional probability, the meaning of these quantities is easily remembered. 

To verify the claims in Table \ref{tab:examples_regression}, it suffices to show that $F$ is quantile calibrated, but neither probabilistically calibrated nor marginally calibrated, for which we refer to Appendix \ref{app:examples}.  The remaining claims then follow from the implications summarized in panel (a) of Figure \ref{fig:hierarchy_regression}.  Note that $F$ does not satisfy the strong version of quantile calibration in \eqref{eq:SQC}, because $q_{1/2}(\cL(Y\mid q_{1/2}(F) = [1,2])) = \{ 1 \}$.
\end{example}

The forecast in the next example is threshold calibrated, but neither quantile calibrated nor probabilistically calibrated.

\begin{example}  \label{ex:TC+nUQC}
Consider a non-equiprobable forecast $F$ for an outcome $Y \in \cY = \{ 1, 2, 3 \}$ as follows:

\begin{center}
\begin{tabular}{>{$}c<{$} | >{$}c<{$} | >{$}c<{$} >{$}c<{$} >{$}c<{$} | >{$}c<{$} >{$}c<{$} >{$}c<{$}}
\toprule
j & \myQ(F = F_j) & f_j(1) & f_j(2) & f_j(3) & g_j(1) & g_j(2) & g_j(3) \\ 
\midrule
1 & 3/5           & 1/2 & 1/4 & 1/4 & 5/10 & 4/10 & 1/10 \\ 
2 & 1/5           & 1/4 & 1/2 & 1/4 & 1/10 & 2/10 & 7/10 \\ 
3 & 1/5           & 1/4 & 1/4 & 1/2 & 4/10 & 1/10 & 5/10 \\ 
\bottomrule
\end{tabular}
\end{center}

\noindent  To verify the claims in Table \ref{tab:examples_regression}, it is enough to show that $F$ is threshold calibrated, but not unconditionally quantile calibrated.  We refer to Appendix \ref{app:examples} for the details.
\end{example}

The notions from Section \ref{sec:classification} for classification tasks do not apply to real-valued outcomes if the forecast distributions have continuous components.  In contrast, the notions for regression studied in this section can be applied whenever the outcomes are ordered in some way (as done with the usual labeling).  We explore the latter in more detail in Section \ref{sec:extended}.

\subsection{A simplified hierarchy for binary targets}  \label{sec:binary}

\begin{figure}[t]  
\centering
\fbox{\subcaptionbox*{Binary setting:}{
\begin{tabular}{cccc}
\midrule
\multicolumn{4}{c}{$\AC \mid \TC \mid \CC \mid \PC \mid \CwCoC \mid \CwC \mid \color{BrickRed}{\DC}$} \\
\midrule
$\Downarrow$ & \rule{8.3mm}{0mm} $\Downarrow$ & \rule{8.3mm}{0mm} $\Downarrow$ & \rule{8.3mm}{0mm} $\Downarrow$ \\ [1mm]
$\CoC$       & \rule{8.3mm}{0mm} $\CMC$ &\rule{8.3mm}{0mm} $\QC$ & \rule{8.3mm}{0mm} $\MC$ \\ [1mm]
             & & \rule{8.3mm}{0mm} $\Downarrow$ & \\ [1mm]
             & & \rule{8.3mm}{0mm} $\UQC$ & \\
\end{tabular}
}}
\caption{Hierarchy of calibration for binary outcomes in terms of the notions and acronyms introduced in Definitions \ref{def:AC_MC}, \ref{def:classification}, and \ref{def:regression}.  Double PIT calibration (\textcolor{BrickRed}{$\DC$}) in \textcolor{BrickRed}{red} color is explored in Section \ref{sec:ordinal}.  \label{fig:hierarchy_binary}}
\end{figure}

We now discuss the special case of a binary outcome, where we may assume that $Y = 1$ represents a success and $Y = 2$ a failure.  In this simple setting, numerous notions of calibration coincide.  The forecast distribution $F$ is now characterized by a single number, usually, the predicted probability of success, $f(1)$.  Then auto-calibration is equivalent to class-wise calibration in the sense that $\myQ(Y = 1 \mid f(1)) = f(1)$, which corresponds to threshold calibration.  Similarly, auto-calibration is equivalent to the strong quantile calibration condition \eqref{eq:SQC} in this setting, whereas Theorem 2.11 of \citet{Gneiting2013} states the equivalence of probabilistic calibration and auto-calibration.  Therefore, auto-calibration, class-wise calibration, class-wise confidence calibration, threshold calibration, conditional (non) exceedance probability calibration, probabilistic calibration, and the strong version of quantile calibration coincide for binary targets.  

Evidently, confidence calibration, modal calibration, and marginal calibration remain strictly weaker than auto-calibration in this setting.  Similarly, the standard version \eqref{eq:QC} of quantile calibration reduces to the condition that
\begin{align}  \label{eq:QC_binary}
\myQ(Y = 1 \mid f(1) < \alpha) \leq \alpha \leq \myQ(Y = 1 \mid f(1) \geq \alpha) 
\end{align}
for $\alpha \in (0,1)$, and the unconditional quantile calibration condition \eqref{eq:UQC} reduces to 
\begin{align}  \label{eq;UQC_binary}
\myQ(f(1) < \alpha) \: \myQ(Y = 1 \mid f(1) < \alpha) \leq \alpha & \\
\leq \myQ(Y = 1 \mid f(1) \geq \alpha) & + \myQ(f(1) < \alpha) \: \myQ(Y = 2 \mid f(1) \geq \alpha) \nonumber
\end{align}
for $\alpha \in (0,1)$, provided the conditional probabilities are well defined, and these conditions are successively weaker.  We summarize our findings in the collapsed hierarchy for the binary setting in Figure \ref{fig:hierarchy_binary}, where the notions in the top row are equivalent.

\section{Bridging notions of calibration for classification and regression}  \label{sec:middle}

As noted, the concepts of auto-calibration and marginal calibration apply universally.  In contrast, the notions from Section \ref{sec:classification} for classification tasks do not apply to real-valued outcomes if the forecast distributions have continuous components.  The notions for regression tasks from Section \ref{sec:regression} can be applied in classification settings if the classes are ordered in some way, though the ordering might be arbitrary.  Naturally, there have been distinct strands of literature that describe concepts of calibration for classification and regression, respectively.  We now aim to connect, bridge, and unify the different strands of literature.  

In Section \ref{sec:DC}, we extend the notion of double PIT calibration \citep{Yang2024} for count data to all types of real-valued outcomes.  For outcomes with a discrete component the extended notion is orthogonal to probabilistic calibration and other types of calibration.  In Section \ref{sec:ordinal}, we consider the general case of ordered data, and in Section \ref{sec:extended} we adapt calibration concepts proposed in regressions settings so that they fit classification settings, by distinguishing full, partial, and average versions of these concepts.  In Section \ref{sec:T}, we generalize results on concepts of calibration that are expressed in terms of properties or functionals of the predictive distributions, such as the modal set in classification settings, and means, quantiles, or threshold (non) exceedance probabilities in regression settings.  While the resulting concepts of calibration are technical in general, a sizable number of the aforementioned notions of calibration fit the latter framework, entailing fresh and insightful interpretations of particular paths in the respective hierarchies.

\subsection{Double PIT calibration}  \label{sec:DC}

We discuss the notion of double  PIT calibration, which was proposed by \citet{Yang2024} specifically for count data, that is, outcomes in the natural numbers.  However, as we demonstrate here, double PIT calibration applies to all kinds of real-valued data.  As before, we let $\myQ$ denote the joint distribution of the random forecast $F$ and the real-valued outcome $Y$.  Furthermore, let $G$ denote the conditional cdf of $Y$ given $F$.  Similar to the mnemonic in Example \ref{ex:dPC}, the meaning of these quantities is easily remembered if one associates the symbol $F$ with the term $f$\/orecast, and the symbol $G$ with the term $g \hsp$enuine conditional cdf.

The motivation of \citet{Yang2024} is that the simplified version $F(Y)$ of the PIT $Z_F$ in \eqref{eq:PIT} agrees with $Z_F$ when $F$ is continuous, and thus is uniform if $F$ is auto-calibrated.  However, $F(Y)$ generally fails to be uniformly distributed when $F$ is discrete, even in cases where the forecast $F$ is auto-calibrated for $Y$. 
Under the random forecast $F$, the hypothesized marginal cdf of the simplified PIT, $F(Y)$, is given by 
\begin{align}  \label{eq:H}
H(\alpha) = \myE_{\hsp \myQ} \left[ \, \myE_{Y \sim F} \, \one \! \left( F(Y) \leq \alpha \right) \hsp \right] = \myE_{\hsp \myQ} \!
\begin{cases}
F(F_+^{-1}(\alpha)-), & \text{if } F(F_+^{-1}(\alpha)) > \alpha, \\
F(F_+^{-1}(\alpha)),  & \text{otherwise},
\end{cases}
\end{align}
for $\alpha \in (0,1)$, where the outer expectation is with respect to the law of $F$ and $F_+^{-1}$ is the upper quantile function \eqref{eq:F_inverse}.  In contrast, the actual cdf of $F(Y)$ under $\myQ$ is given by 
\begin{align}  \label{eq:K}
K(\alpha) 
= \myE_{\hsp \myQ} \left[ \one \left( F(Y) \leq \alpha \right) \right] 
= \myE_{\hsp \myQ} \!
\begin{cases}
 G(F_+^{-1}(\alpha)-), & \text{if } F(F_+^{-1}(\alpha)) > \alpha, \\
 G(F_+^{-1}(\alpha)),  & \text{otherwise},
\end{cases}
\end{align}
for $\alpha \in (0,1)$, where now the expectation is with respect to the joint law of $F$ and $G$.  Motivated by these considerations, we define double PIT calibration as follows.

\begin{definition}  \label{def:DC}
The forecast $F$ for a real-valued outcome $Y$ satisfies \emph{double PIT calibration}\/ ($\DC$) if $H(\alpha) = K(\alpha)$ for all $\alpha \in (0,1)$.
\end{definition}

We collect key properties of double PIT calibration, which are reflected in the hierarchies from Figure \ref{fig:hierarchy_regression}, in the following result.  Note that in part (c) of the following proposition, the assumption of continuity of the law of $F(Y)$ means that the function $K$ in \eqref{eq:K}, which depends on the joint distribution of $F$ and $G$, is continuous.

\begin{proposition}  \label{prop:AC_DC_PC} 
\mbox{}
\begin{enumerate}
\item   \label{prop:AC_DC}
Auto-calibration implies double PIT calibration.
\item  \label{prop:DC_PC}
If\/ $F$ is continuous almost surely, then probabilistic calibration and double PIT calibration are equivalent. 
\item  \label{prop:DC_uniform}
If the law of\/ $F(Y)$ is continuous, then double PIT calibration is equivalent to the uniformity of\/ $H(F(Y))$ on the unit interval.
\end{enumerate}
\end{proposition}

\begin{proof}
The claim in part (a) is immediate, for if $F$ is auto-calibrated then $G = F$ almost surely.  Furthermore, if $F$ is continuous then the PIT \eqref{eq:PIT} equals the simplified PIT, $F(Y)$, which demonstrates part (b).

In order to prove part (c), we note that for a continuous random variable $\widetilde{X}$ and a cdf $\widetilde{F}$, the random variable $\widetilde{F}(\widetilde{X})$ is uniformly distributed on the unit interval if, and only if, $\widetilde{X} \sim \widetilde{F}$. 
The `if' part is stated in many textbooks, see, e.g., \citet[Theorem 2.1.4]{Casella1990}. 
Conversely, if $\widetilde{F}(\widetilde{X})$ is standard uniform, then 
\begin{align*}
\widetilde{F}(x) 
= \myQ \left( \widetilde{F}(\widetilde{X}) < \widetilde{F}(x) \right) 
\leq \myQ( \widetilde{X} < x) \leq \myQ ( \widetilde{X} \leq x) 
= \myQ \left(\widetilde{F}(\widetilde{X}) \leq \widetilde{F}(x) \right) 
= \widetilde{F}(x)
\end{align*}
holds for $x \in \real$, as the monotonicity of $\widetilde{F}$ yields the implications $\widetilde{F}(x) <\widetilde{F}(x') \Rightarrow x < x'$ and $x \leq x' \Rightarrow \widetilde{F}(x) \leq \widetilde{F}(x')$ for $x, x' \in \real$.  Therefore, $\widetilde{X} \sim \widetilde{F}$.  

Suppose now that the law $K$ of $F(Y)$ is continuous.  The above general result implies that $K(F(Y))$ is uniformly distributed on the unit interval.  Therefore, double PIT calibration implies the uniformity of $H(F(Y))$ on the unit interval.  Conversely, if $H(F(Y))$ is uniform on the unit interval, then $F(Y) \sim H$ and thus $H = K$.
\end{proof}

\begin{table}[t]
\caption{Calibration properties for the examples in Section \ref{sec:DC} in terms of the notions in Definitions \ref{def:AC_MC}, \ref{def:regression}, and \ref{def:DC}.  \label{tab:examples_ordinal}} 
\centering 
\footnotesize
\begin{tabular}{lcccccccc} 
\toprule
Forecast                        & $\AC$ & $\CC$ & $\QC$ & $\TC$ & $\PC$ & $\MC$ & $\UQC$ & $\DC$ \\ 
\midrule
Example \ref{ex:dnwPC}, $w = 0$ & --    & --    & --    & --    & --    & --    & --     & \cm \\
Example \ref{ex:cond}           & --    & \cm   & \cm   & \cm   & \cm   & \cm   & \cm    & -- \\
\bottomrule
\end{tabular}
\end{table}

The following examples, which are highlighted in Table \ref{tab:examples_ordinal}, demonstrate that the concept of double PIT calibration is orthogonal to the notions of calibration discussed in Section \ref{sec:regression}.  Example \ref{ex:dnwPC} resembles Example \ref{ex:dPC}, but it is more extreme in a hierarchical sense, as the forecast is double PIT calibrated but neither marginally calibrated nor unconditionally quantile calibrated.  Conversely, the forecast in Example \ref{ex:cond} satisfies all notions from the hierarchies in Figure \ref{fig:hierarchy_regression}, except for auto-calibration and double PIT calibration.

\begin{example}  \label{ex:dnwPC}
Consider the following forecast, where $w \in [0,1)$ is fixed and $\overline{w} = 1 - w$.

\begin{center}
\begin{tabular}{>{$}c<{$} | >{$}c<{$} | >{$}c<{$} >{$}c<{$} >{$}c<{$} >{$}c<{$} | >{$}c<{$} >{$}c<{$} >{$}c<{$} >{$}c<{$}}
\toprule
j & \myQ(F = F_j) & f_j(1) & f_j(2) & f_j(3) & f_j(4) & g_j(1) & g_j(2) & g_j(3) & g_j(4) \\
\midrule
1 & 1/3 & \overline{w}/2 & \! \overline{w}/4 \! & \! \overline{w}/4 \! & w & 0       & \overline{w}/4         & \! \! 3 \, \overline{w}/4 \! \! & w \\ 
2 & 1/3 & \overline{w}/4 & \! \overline{w}/2 \! & \! \overline{w}/4 \! & w & \overline{w}/4 & \overline{w}/2         & \! \overline{w}/4 \!         & w \\ 
3 & 1/3 & \overline{w}/4 & \! \overline{w}/4 \! & \! \overline{w}/2 \! & w & \overline{w}/4 & \! 3 \, \overline{w}/4 \! & 0                     & w \\  
\bottomrule
\end{tabular}
\end{center}

\noindent  In Appendix \ref{app:examples}, we show that $F$ is double PIT calibrated, but fails to satisfy any of the other notions of calibration discussed thus far, when $w = 0$.  If $w > \frac{1}{3}$ then $F$ is class-wise confidence calibrated (as well as confidence calibrated and modally calibrated), which we use later on to show that the extended hierarchy presented in Section \ref{sec:extended} is complete.
\end{example}

\begin{example}  \label{ex:cond}
The forecast in this example is conditional (non) exceedance probability calibrated (as well as threshold calibrated), but it is not double PIT calibrated, as we show in Appendix \ref{app:examples}.
 
\begin{center}
\begin{tabular}{>{$}c<{$} | >{$}c<{$} | >{$}c<{$} >{$}c<{$} >{$}c<{$} | >{$}c<{$} >{$}c<{$} >{$}c<{$}}
\toprule
j & \myQ(F = F_j) & f_j(1) & f_j(2) & f_j(3) & g_j(1) & g_j(2) & g_j(3) \\
\midrule
1 & 1/6 & 3/5 & 1/5 & 1/5 &  3/5 & 3/10 & 1/10 \\ 
2 & 1/6 & 2/5 & 2/5 & 1/5 &  3/5 &  1/5 &  1/5 \\ 
3 & 1/6 & 2/5 & 1/5 & 2/5 &  1/5 &  1/5 &  3/5 \\ 
4 & 1/6 & 1/5 & 3/5 & 1/5 & 1/10 &  3/5 & 3/10 \\ 
5 & 1/6 & 1/5 & 2/5 & 2/5 &  1/5 &  3/5 &  1/5 \\ 
6 & 1/6 & 1/5 & 1/5 & 3/5 & 3/10 & 1/10 &  3/5 \\ 
\bottomrule
\end{tabular}
\end{center}

\end{example}

Finally, we show that for a binary outcome, auto-calibration and double PIT calibration are equivalent.  We adopt notation from Section \ref{sec:binary} and refer to Appendix \ref{app:proofs} for the proof of this result.

\begin{proposition}  \label{prop:binary}
For a binary outcome, if\/ $H(\alpha) = K(\alpha)$ for\/ $\alpha \in (0,1)$ then\/ $f(1) = g(1)$ almost surely. 
\end{proposition}

\subsection{The case of linearly ordered, discrete outcomes}  \label{sec:ordinal}

In this section, we suppose that the outcome space $\cY$ is finite and equipped with a natural total order $\leq$.  In this case, it is conventional to identify $\cY$ with the first $k = \vert \cY \vert$ natural numbers $1,  \ldots, k$.  This convention essentially brings us back to the setting of Section \ref{sec:classification}, except that we now have an order.  

Evidently, the choice of the map from the ordered outcomes to the real line ought to be irrelevant in the assessment of calibration.  In what follows, we let $g: \real \to \real$ be a strictly increasing transformation of $\real$.  We define $F_g$ as the induced cdf that  corresponds to the probability measure $P_{F_g} = P_F \circ g^{-1}$ after transformation of the outcome by $g$.  As $g$ is strictly increasing, $F_g = F \circ g^{-1}$.  The following result shows that the choice of the mapping $g$ is immaterial in the assessment of calibration.  The proof is straightforward and we sketch it only. 

\begin{proposition}  \label{prop:invariance}
Let\/ $g$ be a strictly increasing transformation.  Then, the equivalence
\begin{align}  \label{eq:invariance}
F \text{ satisfies \emph{NC} for } Y \quad \Longleftrightarrow \quad F_g \text{ satisfies \emph{NC} for } g(Y)
\end{align}
holds for all notions\/ $\NC \in \{ \AC, \CC, \DC, \MC, \PC, \QC, \TC, \UQC \}$ from Definitions \ref{def:AC_MC}, \ref{def:regression}, and \ref{def:DC}.
\end{proposition}

\begin{proof}[Sketch of proof.] 
For $\NC \in \{ \AC, \MC, \TC \}$, the claim is obvious.  For $\NC \in \{ \QC, \UQC \}$, the equivalence follows from the equivariance of quantiles under $g$; specifically, $g \circ F^{-1} = F_g^{-1}$ and thus $g(y) < F_g^{-1}(\alpha)$ if, and only if, $y < F^{-1}(\alpha)$.  For $\NC \in \{ \CC, \DC, \PC \}$, the claim follows from the invariance of $Z_F$ and $F(Y)$, respectively, under the transformation.
\end{proof}

Complications arise with double PIT calibration if the number of possible outcomes is finite.  In this case, $F(Y)$ has a point mass at 1 if $Y$ attains the maximum value $k$ with positive probability.  Therefore, the assumption of a continuous distribution for $F(Y)$ in Proposition \ref{prop:AC_DC_PC} \ref{prop:DC_uniform} is violated, and the equivalence of double PIT calibration and the uniformity of the law of $H(F(Y))$ may fail.  However, a variant of the equivalence holds, as follows.  Loosely speaking, double PIT calibration is equivalent to a conditional form of uniformity for $H(F(Y))$ in concert with a weak form of marginal calibration.  We defer the proof of this result to Appendix \ref{app:proofs}.

\begin{proposition}  \label{prop:DC}
In the above setting, let\/ $k \geq 2$ and\/ $\gamma = \myE_{\hsp \myQ} [F(k-1)]$. Suppose that\/ $\myQ \left( F(k-1) < 1 \right) \allowbreak = 1$ and that conditional on\/ $Y \leq k - 1$ the distribution of\/ $F(Y)$ is continuous.  Then\/ $F$ is double PIT calibrated if, and only if, $\myQ(Y \leq k - 1) = \gamma$ and conditional on\/ $Y \leq k - 1$ the random variable\/ $H(F(Y))$ is uniformly distributed on\/ $(0, \gamma)$.
\end{proposition}

For count data, if the law of $F(Y)$ is continuous, to check double PIT calibration in practice,  one can directly examine the uniformity of the histogram of the double PIT according to Proposition \ref{prop:AC_DC_PC} \ref{prop:DC_uniform}.  For ordinal outcomes with a finite number of categories, one ought to assess the condition $\myE_{\hsp \myQ} [F(k-1)] = \myQ( Y \leq k - 1)$ prior to examining the uniformity of the double PIT on the subset of the data for which $Y \leq k - 1$.

The following example demonstrates that the uniformity of $H(F(Y))$ conditional on $Y \leq k - 1$ does not guarantee double PIT calibration, unless furthermore $\myE_{\hsp \myQ} [F(k-1)] = \myQ(Y \leq k - 1)$.

\begin{example}  \label{ex:nDC}
Let $W$ be uniform on $(0,\frac{1}{2})$.  Conditional on $W$, let $F$ assign mass $W$ to $Y = 1$ and mass $1 - W$ to $Y = 2$, while these events materialize with probability $2W$ and $1 - 2W$, respectively, as summarized in the following table.  
\begin{center}
\begin{tabular}{>{$}l<{$} | >{$}c<{$} >{$}c<{$} | >{$}c<{$} >{$}c<{$}}
\toprule
\text{Conditional on}                     & f_W(1) & f_W(2) & g_W(1) & g_W(2) \\
\midrule
W \sim \operatorname{Unif}(0,\frac{1}{2}) & W      & 1 - W  & 2W     & 1 - 2W \\ 
\bottomrule
\end{tabular}
\end{center}

\noindent  Then conditional on $Y = 1$ the random variable $H(F(Y))$ is uniformly distributed, even though $F$ fails to be double PIT calibrated.  For details, see Appendix \ref{app:proofs}.
\end{example}

\begin{remark}
Example \ref{ex:dnwPC} can be adapted to satisfy the assumptions of Proposition \ref{prop:DC} by replacing the fixed class probability $w$ with a suitable random variable $W$ similar to the previous example.  Analogously, we may add an additional class with random probability $W$ in Example \ref{ex:cond}.  We expect the forecast probabilities $\widetilde{f}_j(i) = (1-W) \hsp f_j(i)$ for $i = 1, 2, 3$ and $\widetilde{f}_j(4) = W$ along with the conditional probabilities $\widetilde{g}_j(i) = (1-W) \hsp \hsp g_j(i)$ for $i = 1, 2, 3$ and $\widetilde{g}_j(4) = W$ to preserve the properties stated in Table \ref{tab:examples_ordinal}, while ensuring that the simplified PIT $\widetilde{F}(\widetilde{Y})$ is continuously distributed given $\widetilde{Y} \leq 3$. 
\end{remark}

\subsection{An extended hierarchy for nominal targets}  \label{sec:extended}

We return to the general classification setting of Section \ref{sec:classification}, where the nominal outcome $Y$ belongs to a finite number of classes that fail to be ordered and are labeled arbitrarily as $1, \ldots, k$.  In what follows, we explore how the notions of conditional (non) exceedance probability calibration ($\CC$), threshold calibration ($\TC$), and probabilistic calibration ($\PC$) for ordered real-valued outcomes can be adapted so that they apply to unordered nominal outcomes.  We disregard quantile calibration due to its non-intuitive character and behavior in discrete settings. 

We start by studying calibration properties under permutation of the class labels, where we use notation similar to that in Section \ref{sec:ordinal}.  Let $S_k$ be the symmetric group on $k$ elements, and let $\pi \in S_k$.  Given a predictive cdf $F$ for the outcome $Y$, we denote by $F_\pi$ the predictive cdf for the permuted outcome $Y_\pi$.  As the labels for a nominal variable might be arbitrary, we use qualifiers to refine the aforementioned notions of calibration to hold with respect to \emph{any}\/ or \emph{some}\/ ordering or permutation of the outcomes.

\begin{definition}  \label{def:classification_extended}
Let $\NC \in \{ \CC, \TC, \PC \}$ be a notion of calibration.  The forecast $F$ for the nominal outcome $Y$ is
\begin{enumerate}
\item  \emph{fully}\/ $\NC$ ($\textrm{f}\NC$) if the forecast $F_\pi$ for $Y_\pi$ satisfies $\NC$ for all permutations $\pi \in S_k$;
\item  \emph{partially}\/ $\NC$ ($\operatorname{pNC}$) if the forecast $F_\pi$ for $Y_\pi$ satisfies $\NC$ for some permutation $\pi \in S_k$;
\item  \emph{average}\/ probabilistically calibrated ($\aPC$) if
\begin{align}  \label{eq:aPC}
\frac{1}{k!} \sum_{\pi \in S_k} \myQ(Z_{F_\pi} \leq u) = u \quad \text{for all} \quad u \in [0,1],
\end{align}
where $Z_{F_\pi}$ is the randomized PIT of the forecast $F_\pi$ for $Y_\pi$.
\end{enumerate}
\end{definition}

The notions of full calibration are strictly stronger than their partial counterparts, so the order of the classes matters.  For instance, the forecast in Example 2.4 (b) of \citet{Gneiting2023a} satisfies conditional (non) exceedance probability calibration, threshold calibration, and probabilistic calibration, but fails to be auto-calibrated.  After interchanging class labels 2 and 3, none of these notions is satisfied, nor does average probabilistic calibration hold.  Also, in analogy to the general setting, the full version of threshold calibration does not imply full probabilistic calibration nor full conditional (non) exceedance probability calibration, as Example 2.14 (b) in \citet{Gneiting2023a} illustrates.  

However, the following result shows that the notions considered are unaffected if the order of the classes is reversed.  For the proof, see Appendix \ref{app:proofs}.

\begin{proposition}  \label{prop:permutation}
Let\/ $\rho = \left( \begin{smallmatrix} 1 & 2 & \ldots & k \\ k & k - 1 & \ldots & 1 \end{smallmatrix} \right)$ be the permutation that corresponds to a complete reversal, and let\/ $\NC \in \{ \CC, \TC, \PC \}$ be a notion of calibration.  Then the forecast\/ $F$ for\/ $Y$ satisfies\/ $\NC$ if, and only if, the forecast\/ $F_\rho$ satisfies\/ $\NC$ for\/ $Y_\rho$.  
\end{proposition}

In full generality, if $\pi \in S_k$ is any permutation, then the forecast $F_\pi$ for $Y_\pi$ satisfies $\NC$ if, and only if, the forecast $F_{\rho \circ \pi}$ for $Y_{\rho \circ \pi}$ satisfies $\NC$. Therefore, only half the number of permutations need to be considered in assessing full or partial calibration.  We also note that the qualifiers full, partial and average do not readily generalize to settings where the outcome variable has infinite support.  

Evidently, for all three notions considered in Proposition \ref{prop:permutation}, auto-calibration implies full calibration, and full calibration implies partial calibration and average calibration.  Furthermore, we note the following implication.

\begin{proposition}  \label{prop:fTC_CwCoC}
For a nominal outcome, full threshold calibration implies class-wise calibration.
\end{proposition}

\begin{proof} 
Let $y \in \{ 1, \ldots, k \}$ and suppose that the permutation $\pi \in S_k$ interchanges classes $1$ and $y$.  Then threshold calibration of the forecast $F$ for $Y$ yields
\begin{align*}
f(y) = f_\pi(1) = F_\pi(1) = \myQ( Y_\pi \leq 1 \mid F_\pi(1)) = \myQ( Y_\pi = 1 \mid f_\pi(1)) = \myQ( Y = y \mid f(y)),
\end{align*}
whence $F$ is class-wise calibrated.
\end{proof}

The findings of this section are summarized in the hierarchy presented in Figure \ref{fig:hierarchy_classification_extended}, which substantially extends the hierarchy in Figure \ref{fig:hierarchy_classification}.  Table \ref{tab:examples_classification_extended} previews subsequent examples and counterexamples that show the hierarchy to be complete up to conjectured implications, in the sense that if there is no directed path from a notion to another then there is no implication either.

\begin{figure}[t]
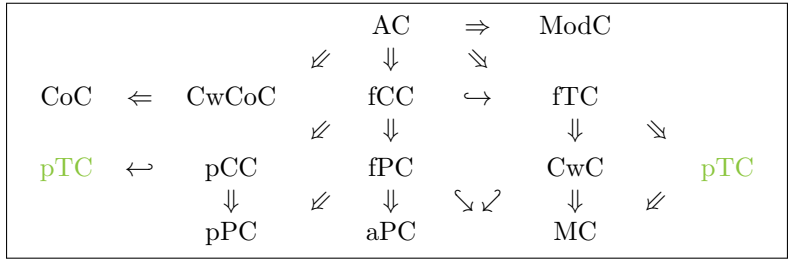
  
\centering
\fbox{
\begin{tabular}{ccccccccc}
& & & & $\AC$ & $\Rightarrow$ & $\CMC$ & & \\
& & & $\SWarrow$ & $\Downarrow$ & $\SEarrow$ & & & \\ [1mm]
$\CoC$ & $\Leftarrow$ & $\CwCoC$ &  & $\fCC$ & $\hookrightarrow$ & $\fTC$ & & \\ 
& & & $\SWarrow$ & $\Downarrow$ && $\Downarrow$ & $\SEarrow$ & \\ [1mm]
\color{LimeGreen}{$\pTC$} & $\hookleftarrow$ & $\pCC$ & & $\fPC$ & & $\CwC$ & & \color{LimeGreen}{$\pTC$} \\
& & $\Downarrow$ & $\SWarrow$ & $\Downarrow$ & $\hookSEarrow\hookSWarrow$ & $\Downarrow$ & $\SWarrow$ & \\
& & $\pPC$ & & $\aPC$ & & $\MC$ & &		
\end{tabular}
}
\caption{Hierarchy of calibration for nominal outcomes in terms of the notions and acronyms introduced in Definitions \ref{def:AC_MC}, \ref{def:classification}, \ref{def:regression}, and \ref{def:classification_extended}.  Hook arrows indicate conjectured implications discussed at the end of this section.  Note the repeat occurrence of partial threshold calibration (\textcolor{LimeGreen}{$\pTC$}) in \textcolor{LimeGreen}{green} color.  \label{fig:hierarchy_classification_extended}}
\end{figure}

\begin{sidewaystable}
\centering
\caption{Calibration properties for nominal outcomes in terms of the notions in Definitions \ref{def:AC_MC}, \ref{def:classification}, and \ref{def:classification_extended} for examples from the extant literature (upper block) and in this paper (lower block).  \label{tab:examples_classification_extended}}
\footnotesize
\begin{tabular}{llccccccccccccc}
\toprule
Example & Specifics & $\AC$ & $\CwCoC$ & $\CoC$ & $\CMC$ & $\CwC$ & $\MC$ & $\fCC$ & $\pCC$ & $\fPC$ & $\pPC$ & $\aPC$ & $\fTC$ & $\pTC$ \\
\midrule
\citet{Vaicenavicius2019}  & Table 1                            & -- & \cm & \cm & \cm & \cm & \cm & --  & \cm & --  & \cm & \cm & \cm & \cm \\
\citet{Vaicenavicius2019}  & Supplement, Table 2                & -- & --  & \cm & \cm & --  & \cm & --  & --  & --  & --  & --  & --  & -- \\
\citet{Silva2023}          & Footnote 2                         & -- & --  & --  & --  & \cm & \cm & --  & \cm & --  & \cm & -- & \cm & \cm \\
\citet{Gneiting2023a}      & Example 2.4 (b)                    & -- & --  & --  & \cm & --  & \cm & --  & \cm & --  & \cm & --  & --  & \cm \\
\citet{Gneiting2023a}      & Example 2.14 (b)                   & -- & \cm & \cm & \cm & \cm & \cm & --  & --  & --  & --  & \cm & \cm & \cm \\
\midrule
Section \ref{sec:AC_MC} & Example \ref{ex:SRS}, $G$ multinomial & -- & --  & --  & --  & --  & \cm & --  & --  & --  & --  & --  & --  & -- \\
Section \ref{sec:ordinal}  & Example \ref{ex:dnwPC}, $w = 1/2$  & -- & \cm & \cm & \cm & --  & --  & --  & --  & --  & --  & --  & --  & -- \\
Section \ref{sec:ordinal}  & Example \ref{ex:cond}              & -- & \cm & \cm & --  & \cm & \cm & \cm & \cm & \cm & \cm & \cm & \cm & \cm \\
Section \ref{sec:extended} & Example \ref{ex:CwC}               & -- & \cm & \cm & --  & \cm & \cm & --  & --  & --  & --  & \cm & --  & -- \\
Section \ref{sec:extended} & Example \ref{ex:pPC+nMC}           & -- & --  & --  & --  & --  & --  & --  & --  & --  & \cm & \cm & --  & -- \\
Section \ref{sec:extended} & Example \ref{ex:fPC+nCoC}          & -- & --  & --  & \cm & --  & \cm & --  & --  & \cm & \cm & \cm & --  & -- \\
Section \ref{sec:extended} & Example \ref{ex:fCC+nCoC}          & -- & --  & --  & \cm & \cm & \cm & \cm & \cm & \cm & \cm & \cm & \cm & \cm \\
\bottomrule
\end{tabular}
\end{sidewaystable}

Our extended hierarchy does not depend on, nor involve, the following result, which proves the existence of a stronger form of Example 2.14 (b) in \citet{Gneiting2023a}, where probabilistic calibration is replaced by full probabilistic calibration.  However, the proof is constructive, based on tools of linear algebra, and it sheds light on how we have generated many of the other examples in this paper, namely, by fixing some simple predictive class probabilities $f_j(1), \ldots, f_j(k)$ and expressing desired notions of calibration as linear constraints on the respective conditional class probabilities $g_j(1), \ldots, g_j(k)$ as detailed in Appendix \ref{app:examples}.  In this light, we believe that the proposition and its proof are of interest in their own right.

\begin{proposition}  \label{prop:constructive}
There exists a forecast--outcome pair\/ $(F,Y)$ such that\/ $F$ is fully probabilistically calibrated and marginally calibrated, but fails to be auto-calibrated for the outcome\/ $Y$.
\end{proposition}

\begin{proof}
Let the outcome space be the set $\{ 1, \ldots, k \}$ and let $N$ be the number of distinct forecasts.  We encode the forecast probabilities $f_j = (f_j(1), \ldots, f_j(k))$ and conditional class probabilities $g_j = (g_j(1), \ldots, g_j(k)) = (\myQ( Y = i \mid F = F_j)_{i=1}^k$ as vectors from the unit simplex $\Delta_k = \{ v \in \real^k_{\geq 0} : v_1 + \cdots + v_k = 1 \}$.  For $i = 1, \ldots, k$ and $j = 1, \ldots, N$, we restrict the forecast probabilities $f_j(i)$ to probabilities of the form $m/n$, where $m$ is a positive integer.  Thus, the function $u \mapsto \myQ(Z_{F_\pi} \leq u)$ is piecewise linear with $n - 1$ possible breakpoints at $1/n, \ldots, (n-1)/n$ for any permutation $\pi$ of the outcome space.  Assuming $n \geq k$ in what follows, the number of distinct forecast vectors is given by
\begin{align*}
N = \binom{k + (n - k) - 1}{n - k} = \binom{n - 1}{n - k}, 
\end{align*}
the number of $(n-k)$-combinations with repetition.  Assuming that the forecast vectors occur with equal probability, probabilistic calibration for a given permutation $\pi$ of the classes is equivalent to a set of $n - 1$ linear equations being satisfied by the conditional class probabilities $g_j(i)$, where $i = 1, \ldots, k$ and $j = 1, \ldots, N$.  Since there are $k!$ permutations of the classes, full probabilistic calibration yields a total of $k! \, (n-1)$ linear constraints on the true conditional probabilities.  In addition, marginal calibration yields $k$ linear equations for a total of 
\begin{align}  \label{eq:nLinEq}
k! \, (n-1) + k
\end{align}
constraints.

Subject to these constraints, we need to choose $N$ vectors, $g_1, \ldots, g_N \in \Delta_k$, that is, for each $j = 1, \ldots, N$, we choose $g_j(1), \ldots, g_j(k-1) \geq 0$ such that $\sum_{i = 1}^{k-1} g_j(i) \leq 1$ and set $g_j(k) = 1 - \sum_{i = 1}^{k-1} g_j(i)$.  Thus, the vectors $g_1, \ldots, g_N$ are completely determined by the projection to the first $k - 1$ components,
\begin{align*}
\pi_{k-1} \colon \real^k \to \real^{k-1}, \quad (x_1, \ldots, x_k) \longmapsto (x_1, \ldots, x_{k-1}),
\end{align*}
and, since the relation with the last component is linear, we can formulate each previous linear constraint using only the first $k - 1$ components of each $g_j$.  This yields a non-homogeneous system of 
linear equations in
\begin{align}  \label{eq:nVar}
(k-1) \, N = (k-1) \, \binom{n-1}{n-k}
\end{align}
variables, namely, the conditional class probabilities $g_j(i)$, where $i = 1, \ldots, k - 1$ and $j = 1, \ldots, N$.
	
Fortunately, there is an obvious solution to the system, expressly, 
\begin{align*} 
a = (f_1(1), \ldots, f_1(k-1), \, f_2(1), \ldots, f_2(k-1), \ldots, f_N(1), \ldots, f_N(k-1)) \in (\pi_{k-1}(\Delta_k))^N \subset \real^{(k-1)N},
\end{align*}
which we obtain by setting every conditional class probability equal to the respective forecast probability, that is, $g_j(i) = f_j(i)$ for $i = 1, \ldots, k - 1$ and $j = 1, \ldots, N$.  This solution corresponds to auto-calibration, which implies full probabilistic calibration and marginal calibration, and thus fulfills all linear constraints.  The auto-calibrated solution $a$ lies in the interior of the closed subset $(\pi_{k-1}(\Delta_k))^N \subset \real^{(k-1)N}$ by construction, as points on the boundary of the unit simplex are characterized by having a vanishing component, and thus there exists a ball $B(a)$ around $a$ that is contained in $(\pi_{k-1}(\Delta_k))^N$.  If the rank $R$ of the linear system is less than the number $(k-1)N$ of variables, the solution space is an affine subspace of dimension $(k-1)N - R > 0$ that contains additional solutions in $\real^{(k-1)N}$.  In this case, the intersection of the solution space and the closed set $(\pi_{k-1}(\Delta_k))^N$ needs to contain further solutions (which then fail to be auto-calibrated), as it includes the interior point $a$ of said set (more precisely, the affine space cuts through the center of the ball $B(a)$ and thus contains additional points from the ball as it also has to cut through its hull).  Since, for fixed $k \geq 3$, the number $(k - 1) N$ of variables in \eqref{eq:nVar} grows faster in $n$ than the number of linear constraints in \eqref{eq:nLinEq}, which is an upper bound for the rank $R$ of the linear system, the proof is complete.
\end{proof}

While the proof may give the impression that the number of forecasts $N$ needs to be quite large to construct explicit examples, it turns out that the rank of the linear system tends to be much lower than the upper bound in \eqref{eq:nLinEq}.  For instance, Proposition \ref{prop:permutation} shows that only half of the permutations from the definition of full probabilistic calibration yield any additional constraints.  Consequently, even when additional constraints are introduced to satisfy further notions of calibration, a small number $N$ of distinct forecast distributions may suffice, as illustrated by Example \ref{ex:cond}, where $k = 3, n = 5$, and $N = 6$.  

The following examples show that the hierarchy in Figure \ref{fig:hierarchy_classification_extended} is not missing any arrows.

\begin{example}  \label{ex:CwC}
Consider $k = 4$ classes with forecast probabilities and conditional outcome probabilities as follows.

\begin{center}
\begin{tabular}{>{$}c<{$} | >{$}c<{$} | >{$}c<{$} >{$}c<{$} >{$}c<{$} >{$}c<{$} | >{$}c<{$} >{$}c<{$} >{$}c<{$} >{$}c<{$}}
\toprule
j & \myQ(F = F_j) & f_j(1) & f_j(2) & f_j(3) & f_j(4) & g_j(1) & g_j(2) & g_j(3) & g_j(4) \\
\midrule
1 & 1/4 & 2/5 & 1/5 & 1/5 & 1/5 &  2/5 & 1/10 &  2/5 & 1/10 \\ 
2 & 1/4 & 1/5 & 2/5 & 1/5 & 1/5 & 1/10 &  2/5 & 1/10 &  2/5 \\ 
3 & 1/4 & 1/5 & 1/5 & 2/5 & 1/5 & 1/10 &  2/5 &  2/5 & 1/10 \\ 
4 & 1/4 & 1/5 & 1/5 & 1/5 & 2/5 &  2/5 & 1/10 & 1/10 &  2/5 \\ 
\bottomrule
\end{tabular}
\end{center}

\noindent  Evidently, this forecast is class-wise confidence calibrated.  However, in Appendix \ref{app:examples} we show that it is neither partially threshold calibrated nor partially probabilistically calibrated.  In particular, the implication in Proposition \ref{prop:fTC_CwCoC} cannot be reversed.
\end{example}

\begin{example}  \label{ex:pPC+nMC}
This example shares the forecast values with Example 2.4 (b) in \citet{Gneiting2023a}.  It demonstrates that partial probabilistic calibration and average probabilistic calibration do not imply marginal calibration.  

\begin{center}
\begin{tabular}{>{$}c<{$} | >{$}c<{$} | >{$}c<{$} >{$}c<{$} >{$}c<{$} | >{$}c<{$} >{$}c<{$} >{$}c<{$}}
\toprule
j & \myQ(F = F_j) & f_j(1) & f_j(2) & f_j(3) & g_j(1) & g_j(2) & g_j(3) \\
\midrule
1 & 1/3 & 1/2 & 1/4 & 1/4 &  2/3 & 1/3 &    0 \\ 
2 & 1/3 & 1/4 & 1/2 & 1/4 & 5/12 & 1/6 & 5/12 \\ 
3 & 1/3 & 1/4 & 1/4 & 1/2 &    0 & 1/3 &  2/3 \\
\bottomrule
\end{tabular}
\end{center}

\end{example}

\begin{example}  \label{ex:fPC+nCoC}
In the example given by the following table, the forecast is fully probabilistically calibrated, but fails to be confidence calibrated, as 
\begin{align*}
\myQ \left( Y \in \hat{y}_F \left| \, m_F = \frac{2}{3} \right. \right) = \frac{1}{3} \left( \frac{14}{15} + \frac{2}{3} + \frac{4}{5} \right) = \frac{4}{5}.
\end{align*}

\begin{center}
\begin{tabular}{>{$}c<{$} | >{$}c<{$} | >{$}c<{$} >{$}c<{$} >{$}c<{$} | >{$}c<{$} >{$}c<{$} >{$}c<{$}}
\toprule
j & \myQ(F = F_j) & f_j(1) & f_j(2) & f_j(3) & g_j(1) & g_j(2) & g_j(3) \\
\midrule
1 & 1/6 & 2/3 & 1/6 & 1/6 & 14/15 & 1/30 & 1/30 \\ 
2 & 1/6 & 1/2 & 1/3 & 1/6 &   1/5 & 8/15 & 4/15 \\ 
3 & 1/6 & 1/6 & 2/3 & 1/6 &   1/6 &  2/3 &  1/6 \\ 
4 & 1/6 & 1/6 & 1/2 & 1/3 &  4/15 &  1/5 & 8/15 \\ 
5 & 1/6 & 1/6 & 1/3 & 1/2 &  4/15 & 8/15 &  1/5 \\ 
6 & 1/6 & 1/6 & 1/6 & 2/3 &     0 &  1/5 &  4/5 \\ 
\bottomrule
\end{tabular}
\end{center}
        
\end{example}

Finally, the next example shows that full conditional (non) exceedance calibration does not imply confidence calibration.  

\begin{example}  \label{ex:fCC+nCoC}
The example with 13 equiprobable forecasts described in the following table is fully conditional (non) exceedance calibrated but not confidence calibrated.  Except for the first two columns, the entries are stated in multiples of one over a thousand. 

\begin{center}
\begin{tabular}{ >{$}c<{$} | >{$}c<{$} | >{$}r<{$} >{$}r<{$} >{$}r<{$} | >{$}r<{$} >{$}r<{$} >{$}r<{$} }
\toprule
j & \myQ(F = F_j) & f_j(1) & f_j(2) & f_j(3) & g_j(1) & g_j(2) & g_j(3) \\
\midrule
 1 & 1/13 & 200 & 300 & 500 & 214 & 321 & 465 \\ 
 2 & 1/13 & 700 & 100 & 200 & 700 & 163 & 137 \\ 
 3 & 1/13 & 200 & 700 & 100 & 186 & 651 & 163 \\ 
 4 & 1/13 & 300 & 400 & 300 &  48 & 484 & 468 \\ 
 5 & 1/13 & 100 & 400 & 500 & 177 & 288 & 535 \\ 
 6 & 1/13 & 300 & 500 & 200 & 552 & 185 & 263 \\ 
 7 & 1/13 & 200 & 500 & 300 & 158 & 710 & 132 \\ 
 8 & 1/13 & 500 & 100 & 400 & 395 &  79 & 526 \\ 
 9 & 1/13 & 100 & 700 & 200 &   2 & 749 & 249 \\ 
10 & 1/13 & 500 & 400 & 100 & 535 & 428 &  37 \\ 
11 & 1/13 & 200 & 100 & 700 & 242 &  58 & 700 \\ 
12 & 1/13 & 500 & 300 & 200 & 570 & 279 & 151 \\ 
13 & 1/13 & 100 & 500 & 400 & 121 & 605 & 274 \\ 
\bottomrule
\end{tabular}
\end{center}
 
\end{example}

To summarize, the hierarchy in Figure \ref{fig:hierarchy_classification_extended} is complete up to the hook arrows, which represent conjectured implications.  Specifically, we conjecture that full probabilistic calibration implies marginal calibration and that class-wise calibration implies average probabilistic calibration.  The conjectures that full conditional (non) exceedance probability calibration implies full threshold calibration, and that partial conditional (non) exceedance probability calibration implies partial threshold calibration, are inherited from their counterparts for general real-valued outcomes in Figure \ref{fig:hierarchy_regression}.

\subsection{Conditioning on functionals: (Multi-)calibration for distributional properties}  \label{sec:T}

There is much ongoing work on the notion of multicalibration, which is a term coined by \citet{HebertJohnson2018}.  While the original idea of multicalibration is to condition on a covariate or feature, which is crucial when discussing fairness, similar ideas have featured prominently in the literature on backtests for financial risk measures \citep{Nolde2017}. In many cases, the covariate can be bypassed when interest centers on calibration.  Specifically, if one conditions on a distributional property or functional $\myT$ of the forecast $F$, the notion of multicalibration reduces to calibration and fits our framework --- namely, the study of concepts of calibration that can be expressed in terms of the joint distribution $\myQ$ of the random tuple $(F,Y)$.  In this setting, there are substantial overlaps between the independently --- and simultaneously --- developed advanced theories of multicalibration \citep{Noarov2023} and $\myT$-calibration \citep{Gneiting2023a}.  In a further, very recent development, \citet{Derr2025} introduce the related, yet stronger notion of distribution calibration with respect to a distributional property or functional $\myT$.  Before we proceed to definitions and link to the notions and examples from earlier sections in the paper, we note that $\myT$-calibration manifests the self-realization interpretation of calibration \citep{Derr2025}.

Tending to the technical development, let $\cP$ be a convex class of Borel probability measures on the real line.  A functional is any mapping $\myT$ that assigns a set $\myT(P) \subseteq \real$ to a distribution $P \in \cP$.  Frequently, the set $\myT(P) = [\hsp \myT^-(P), \myT^+(P)]$ is a closed interval.  In the special case of a \emph{singleton}\/ functional, $\myT(P) = \myT^-(P) = \myT^+(P) \in \real$ is a single number, and we simplify the notation accordingly for ease of exposition.  For example, the mean 
\begin{align}  \label{eq:mu}
\mu_P = \int x \, \dd P(x)
\end{align}
corresponds to the singleton functional $P \mapsto \myT(P) = \mu_P$ on the class $\cP$ of the probability measures with finite first moment.  In contrast, there are functionals of practical relevance that require a set-valued approach.  For instance, the interval-valued full $\alpha$-quantile functional \eqref{eq:q_alpha} is a non-degenerate, closed interval if the associated cdf has a horizontal segment at the level $\alpha \in (0,1)$, as is unavoidable in the ubiquitous setting of Example \ref{ex:SRS}.  Similarly, the mode functional $\hat{y}_P$ in \eqref{eq:mode} fails to be singleton if multiple classes are predicted to occur with the modal probability $m_P$. 

In the subsequent definitions, we follow \citet{Gneiting2023a} and consider settings where the outcome $Y$ is real-valued and the functional $\myT$ maps to subsets of the real line.  Furthermore, we make the tacit assumptions that the convex class $\cP$ is a standard Borel space generated by the topology of weak convergence and that the functional $\myT$ is a measurable map into a standard Borel space. For details on the measurability of functionals and the Borel structure on $\cP$, we refer the interested reader to \citet{Fissler2022}. Finally, we suppose that the class $\cP$ is rich enough to contain all relevant distributions. Specifically, we suppose that the random forecast $F$ belongs to $\cP$ almost surely, and we assume the existence of a regular version $G_{\myT(P_F)}$ of the conditional distribution $\cL( Y \mid \myT(P_F) )$ that belongs to the class $\cP$ almost surely.  In analogy to mnemonic interpretations in earlier sections, the symbol $F$ stands for the random $f$\/orecast and the symbol $G_{\myT(P_F)}$ for the $g \hsp$enuine conditional distribution given the distributional property $\myT(P_F)$.

\begin{definition}  \label{def:T}
The forecast $F$ for $Y$ is 
\begin{enumerate}
\item  \emph{distribution calibrated}\/ with respect to the functional $\myT$ ($\DC_{\myT}$) if
\begin{align}  \label{eq:DC_T}
G_{\myT(P_F)}(B) = \myE_{\hsp \myQ} \left( P_F(B) \mid \myT(P_F) \right) 
\end{align}
almost surely for all Borel sets $B \subseteq \real$;
\item  \emph{conditionally}\/ $\myT$-\emph{calibrated}\/ ($\CC_{\myT}$) if 
\begin{align}  \label{eq:CC_T}
\myT(G_{\myT(P_F)}) = \myT(P_F) 
\end{align} 
almost surely.
\end{enumerate}
\end{definition}

Following \citet[Definition 12]{Frongillo2021}, we say that the functional $\myT$ \emph{refines}\/ the functional $\myT'$ if $\myT' = h \circ \myT$ for some measurable function $h$.  Evidently, any functional refines a constant functional.  Auto-calibration and marginal calibration arise at the extreme ends of distribution calibration.  Specifically, marginal calibration is distribution calibration with respect to the least refined property, namely, a constant functional, and auto-calibration is distribution calibration with respect to the most refined property, namely, the identity functional.\footnote{Technically, we restrict attention to functionals that map to subsets of the real line, so the identity functional is excluded.  However, the notions in Definition \ref{def:T} extend to functionals that map into arbitrary Borel spaces.  We leave these further developments to future work.}  Table \ref{tab:examples_T} lists notions that arise as special cases of conditional $\myT$-calibration relative to either a single functional or a family of functionals.

\begin{table}[t]
\caption{Notions of calibration that admit an interpretation in terms of the general notion \eqref{eq:CC_T} of conditional $\myT$-calibration $(\CC_{\myT})$.  \label{tab:examples_T}}
\centering
\begin{tabular}{lcccc}
\toprule
Section                           & Eq.            & Acronym & $\myT(P_F)$      & For all \\
\midrule
Section \ref{sec:AC_MC}          & \eqref{eq:AC}  & $\AC$   & $F$              & -- \\
Section \ref{sec:AC_MC}          & \eqref{eq:MC}  & $\MC$   & constant         & -- \\
\midrule 
Section \ref{sec:classification} & \eqref{eq:CwC} & $\CwC$  & $f(y)$           & $y \in \{ 1, \ldots, k \}$ \\
Section \ref{sec:classification} & \eqref{eq:CMC} & $\CMC$  & $\hat{y}_F$      & -- \\
\midrule 
Section \ref{sec:regression}     & \eqref{eq:TC}  & $\TC$   & $F(y)$           & $y \in \real$ \\
Section \ref{sec:regression}     & \eqref{eq:QC}  & $\QC$   & $F^{-1}(\alpha)$ & $\alpha \in (0,1)$ \\
Section \ref{sec:regression}     & \eqref{eq:SQC} & $\SQC$  & $q_\alpha(F)$    & $\alpha \in (0,1)$ \\
\midrule 
Section \ref{sec:T}              & \eqref{eq:EC}  & $\EC$   & $\mu_F$          & -- \\
\bottomrule
\end{tabular}
\end{table}

The following definition operates under regularity conditions from \citet{Gneiting2023a} that are satisfied in many practically relevant settings.  Specifically, the functional $\myT$ is interval-valued if $\myT(P) = [\hsp \myT^-(P), \myT^+(P)]$, where  $\myT^-(P) \leq \myT^+(P)$, is a possibly degenerate, closed interval.  The interval-valued functional $\myT(P)$ is \emph{identifiable}\/ relative to the convex class $\cP$ of probability measures if there exists an \emph{identification function}\/ $V : \real \times \real \to \real$ such that $V( \cdot, y)$ is increasing and left-continuous for all $y \in \real$, and $\myE_{Z \sim P} [V(\myT^-(P) + \varepsilon, Z)] \geq 0$ and $\myE_{Z \sim P} [ V(\myT^+(P) - \varepsilon, Z)] \leq 0$ for all $\varepsilon > 0$ and all fixed $P \in \cP$.  Following \citet[Assumption 2.8]{Gneiting2023a}, an identification function $V$ for $\myT$ is of \emph{prediction error}\/ form if there is an increasing, left-continuous function $v \colon \real \to \real$ such that $V(x,y) = v(x-y)$ with $v(-r) < 0$ and $v(r) > 0$ for some $r > 0$, and of \emph{expectation}\/ form if $V(x,y) = x - \myT(\delta_y)$ for a singleton functional $\myT$, where $\delta_y$ denotes the Dirac measure in $y \in \real$.  Table 1 of \citet{Gneiting2023a} lists examples of identifiable functionals and associated identification functions of prediction error or expectation form, including the mean- or expectation functional \eqref{eq:mu}, for which $V(x,y) = x - y$, moments, threshold (non) exceedances, quantiles, and expectiles.

\begin{definition}  \label{def:T_unconditional_1}
Let the functional $\myT = [\hsp \myT^-, \myT^+]$ be interval-valued and identifiable with identification function $V$ that is of prediction error or expectation form.  Then the forecast $F$ for $Y$ is \emph{unconditionally}\/ $\myT$-\emph{calibrated}\/ ($\UC_{\myT}$) if 
\begin{align}  \label{eq:UC_T}
\myE_{\hsp \myQ} \left[ V(\myT^-(P_F) + \varepsilon, Y) \right] \geq 0 \quad \text{and} \quad 
\myE_{\hsp \myQ} \left[ V(\myT^+(P_F) - \varepsilon, Y) \right] \leq 0 \quad \text{for all} \quad \varepsilon > 0.	
\end{align}
\end{definition}

As an illustration of conditional and unconditional $\myT$-calibration, consider mean- or expectation calibration, as studied by \citet{Jung2021}, \citet{Kruger2021a}, and  \citet{Gneiting2023a}.  With $\myT(P) = \mu_P$ denoting the mean- or expectation functional \eqref{eq:mu}, conditional mean- or expectation calibration ($\EC$) requires that
\begin{align}  \label{eq:EC}
\myE_{\hsp \myQ} \left[ \hsp Y \mid \mu_F \right] = \mu_F
\end{align}
almost surely.  Taking expectations in \eqref{eq:EC} yields unconditional mean-calibration, namely, $\myE_{\hsp \myQ} \left[ \hsp Y \right] = \myE_{\hsp \myQ} \left[ \hsp\hsp \mu_F \right]$, which agrees with the classical unbiasedness condition and with the formulation in \eqref{eq:UC_T} in terms of the identification function $V(x,y) = x - y$.

We now return to the theme of hierarchies of calibration.  A preview of the subsequent results is given in Figure \ref{fig:hierarchy_T}, subject to minimal conditions that ensure the existence of regular conditional distributions, as laid out above.  Our results generalize findings in  \citet{Gneiting2023a}, \citet{Noarov2023}, and \citet{Derr2025}, none of which cover the set-valued mode functional.  While \citet{Gneiting2023a} focus on identifiable functionals, the results in \citet{Noarov2023} and \citet{Derr2025} apply to single-valued functionals only.  Furthermore, extant results typically employ further restrictive assumptions, requiring the continuity of functionals \citep{Noarov2023} or discrete prediction spaces \citep{Derr2025}.  Our measure theoretic setting allows for a more general treatment that replaces the convex level sets (CxLS) property in \citet{Noarov2023} and \citet{Derr2025} with the requirement that the level sets $\LS_{\myT}(r) = \{ P \in \cP : \myT(P) = r \}$ of $\myT$ are closed under mixtures.  Specifically, we say that a functional $\myT$ has the \emph{mixLS}\/ property if, given any Borel probability measure $\Lambda$ on the class $\cP$ such that $\Lambda(\LS_{\myT}(r)) = 1$ for some $r \subseteq \real$, it holds that $Q = \int P \hsp \dd \Lambda(P) \in \cP$ and $\myT(Q) = r$.

A functional $\myT$ is \emph{elicitable}\/ if there exists a scoring function $s$ on $\real$ such that $ \myE_P \hsp [ s(t,Y) ] \leq \myE_P \hsp [s(t',Y) ]$ for all fixed probability measures $P \in \cP$, all $t \in \myT(P)$, and all $t' \in \real$, with equality only if $t' \in \myT(P)$.  Many, but not all, functionals are elicitable, including the aforementioned cases of the mean- or expectation functional, moments, threshold (non) exceedances, quantiles, and expectiles \citep{Gneiting2011, Steinwart2014}.  The following proposition shows that our subsequent results apply to elicitable functionals in particular.  For the proof, see Appendix \ref{app:proofs}.

\begin{proposition}  \label{prop:mixLS}
Every elicitable functional\/ $\myT$ extends to a sufficiently large class\/ $\cP$ that admits the mixLS property.
\end{proposition}

\begin{figure}[t]
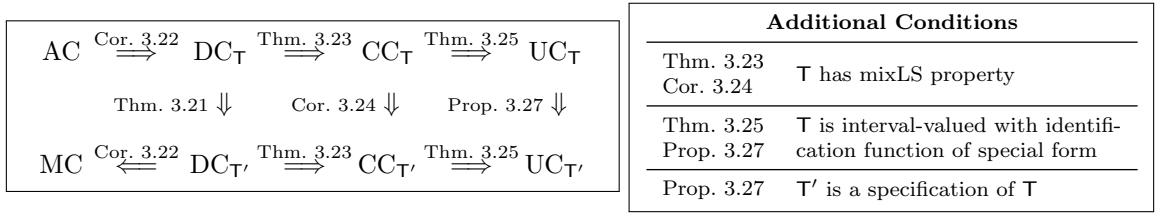

\centering
\fbox{
\begin{tabular}{ccccccc}
$\AC$ & $\overset{\mathclap{\text{Cor.\ \ref{cor:T}}}}{\Longrightarrow}$ & $\DC_{\myT}$ & $\overset{\mathclap{\text{Thm.\ \ref{thm:T}}}}{\Longrightarrow}$ & $\CC_{\myT}$ & $\overset{\mathclap{\text{Thm.\ \ref{thm:T_unconditional}}}}{\Longrightarrow}$ & $\UC_\myT$ \\ [2.5mm]
& & $\mathllap[\scriptstyle]{\text{Thm.\ \ref{thm:refinement}}}\Downarrow$ & & $\mathllap[\scriptstyle]{\text{Cor.\ \ref{cor:refinement}}}\Downarrow$ & & $\mathllap[\scriptstyle]{\text{Prop.\ \ref{prop:refinement}}}\Downarrow$ \\ [3mm]
$\MC$ & $\overset{\mathclap{\text{Cor.\ \ref{cor:T}}}}{\Longleftarrow}$ & $\DC_{\myT'}$ & $\overset{\mathclap{\text{Thm.\ \ref{thm:T}}}}{\Longrightarrow}$ & $\CC_{\myT'}$ & $\overset{\mathclap{\text{Thm.\ \ref{thm:T_unconditional}}}}{\Longrightarrow}$ & $\UC_{\myT'}$
\end{tabular}
}
\fbox{
\footnotesize
\begin{tabular}{ll}
\multicolumn{2}{c}{\textbf{Additional Conditions}} \\
\midrule
\makecell[l]{Thm.\ \ref{thm:T} \\ Cor.\ \ref{cor:refinement}} & $\myT$ has mixLS property \\
\midrule
\makecell[l]{Thm.\ \ref{thm:T_unconditional} \\ Prop.\ \ref{prop:refinement}} & \makecell[l]{$\myT$ is interval-valued with identifi- \\ cation function of special form} \\
\midrule
\makecell[l]{Prop.\ \ref{prop:refinement}} & $\myT'$ is a specification of $\myT$ \\
\end{tabular}
}
\caption{Hierarchy of calibration in terms of the notions and acronyms introduced in Definition \ref{def:T}, where $\myT' = h\circ \myT$ is less refined than $\myT$.  Nontrivial assumptions are listed on the right-hand side.  \label{fig:hierarchy_T}}
\end{figure}
  
The following result is motivated by a related implication in Proposition 4 in \citet{Derr2025} and does not depend on the mixLS property.

\begin{theorem}  \label{thm:refinement} 
Let\/ $\myT$ be a functional, and let\/ $\myT' = h \circ \myT$ be a less refined functional.  Then distri\-bution calibration with respect to\/ $\myT$ implies distribution calibration with respect to\/ $\myT'$.
\end{theorem}

\begin{proof}
By standard properties of measurable functions, the $\sigma$-algebras $\sigma(\myT'(P_F)) = \sigma(h \hsp \circ \myT(P_F)) \subset \sigma(\myT(P_F))$ are nested.  There\-fore, the tower property yields
\begin{align*}
\myQ \left( Y \in B \mid \myT'(P_F) \right) 
& = \myE_{\hsp \myQ} \left[ \one ( Y \in B )  \mid \myT'(P_F) \hsp \right] \\
& = \myE_{\hsp \myQ} \left[ \, \myE_{\hsp \myQ}[ \one ( Y \in B ) \mid \myT(P_F) ] \mid \myT'(F) \hsp \right] \\ 
& = \myE_{\hsp \myQ} \left[ \, \myE_{\hsp \myQ}[ P_F(B) \mid \myT(P_F) ] \mid \myT'(P_F) \hsp \right] \\
& = \myE_{\hsp \myQ} \left[ P_F(B) \mid \myT'(P_F) \hsp \right]
\end{align*}
for all Borel sets $B \subseteq \real$.
\end{proof}

The next result directly follows when considering the most or least refined functional corresponding to auto-calibration or marginal calibration, respectively.  Its findings are obvious, but important, and so we state them explicitly.

\begin{corollary}  \label{cor:T}
Auto-calibration implies distribution calibration with respect to\/ $\myT$, and distribution calibration with respect to\/ $\myT$ implies marginal calibration.
\end{corollary}

The following theorem illustrates that distributional calibration is a stronger notion of calibration than conditional calibration, subject to the aforementioned mixLS property.  The result generalizes Proposition 6 in \citet{Derr2025} and, together with the previous corollary, refines and generalizes the first part of Theorem 2.11 in \citet{Gneiting2023a} as well as Theorem 3.6 in \citet{Noarov2023}.

\begin{theorem}  \label{thm:T}
Suppose that the functional\/ $\myT$ has the mixLS property.  Then distribution calibration with respect to\/ $\myT$ implies conditional\/ $\myT$-calibration.
\end{theorem}

\begin{proof}
The disintegration theorem \citep[Theorem 1]{Chang1997} guarantees the existence of a regular conditional distribution $\Lambda_{\myT(P_F)}$ of the forecast $F$ given $\myT(P_F)$ on the standard Borel space $\cP$.  The assumption of distributional calibration of $F$ for $Y$ with respect to $\myT$ then implies that $F_\myT = \int H \hsp \dd \Lambda_{\myT(P_F)}(H)$ is a regular conditional distribution of $Y$ given $\myT(P_F)$, that is, $F_\myT = G_{\myT(P_F)}$ almost surely.  By the mixLS property of $\myT$ and the concentration property \citep[Definition 1(i)]{Chang1997}, we have $\myT(G_{\myT(P_F)}) = \myT(F_\myT) = \myT(P_F)$ almost surely.  Therefore, the forecast $F$ is conditionally $\myT$-calibrated.
\end{proof}

As a corollary, we obtain an analogue to Theorem \ref{thm:refinement} for conditional calibration, which is subject to the mixLS property and generalizes Proposition 13 in \citet{Derr2025}.

\begin{corollary}  \label{cor:refinement}
Let\/ $\myT, \myT'$ be two functionals such that\/ $\myT' = h \circ \myT$ is a less refined functional.  If\/ $\myT$ has the mixLS property, conditional\/ $\myT$-calibration implies conditional\/ $\myT'$-calibration.
\end{corollary}

\begin{proof}
If $F$ is conditionally $\myT$-calibrated, $\myT(P_F) = \myT(G_{\myT(P_F)})$ almost surely.  In view of Definition \ref{def:T}, we conclude that $G_{\myT(P_F)}$ is distributionally $\myT$-calibrated.  By Theorem \ref{thm:refinement}, $G_{\myT(P_F)}$ is distributionally $\myT'$-calibrated, and by Theorem \ref{thm:T} it is conditionally $\myT'$-calibrated.  Therefore,
\begin{align*}
\myT'(P_F) = h \circ \myT(P_F) = h \circ \myT(G_{\myT(P_F)}) = \myT'(G_{\myT(P_F)}) = \myT'(G_{\myT'(P_F)})
\end{align*}
almost surely, that is, the forecast $F$ is conditionally $\myT'$-calibrated.
\end{proof}

Adopting terminology from \citet{Fissler2021}, we call a singleton functional $\myT'$ that is refined by a set-valued functional $\myT$ a \emph{specification}\/ of $\myT$ if $\myT'(P) \in \myT(P)$ for all $P \in \cP$.  For example, the transition from the full quantile functional $\myT(P_F) = q_\alpha(F) = [F^{-1}(\alpha), F_+^{-1}(\alpha)]$ in \eqref{eq:q_alpha} to $\myT'(P_F) = \myT^-(P_F) = F^{-1}(\alpha)$ is a specification.  Evidently, for a singleton functional, there is a unique specification, and we do not distinguish between the functional and its specification.  As an important special case, the corollary shows that any specification of a functional with the mixLS property --- and, in particular, any specification of an elicitable functional --- preserves conditional $\myT$-calibration.  The corollary thus also generalizes the first part of Proposition 2.18 of \citet{Gneiting2023a}, which concerns specifications that correspond to the endpoints of interval-valued, identifiable functionals.

For completeness, we extract the following result from Theorem 2.11 of \citet{Gneiting2023a}.

\begin{theorem}  \label{thm:T_unconditional}
Suppose that the functional\/ $\myT$ satisfies the assumptions of Definition \ref{def:T_unconditional_1}.  Then conditional\/ $\myT$-calibration implies unconditional\/ $\myT$-calibration.
\end{theorem}

Finally, we note that unconditional $\myT$-calibration is preserved under specification. 

\begin{definition}  \label{def:T_unconditional_2}
Let $\myT'$ be a specification of a functional $\myT$ that satisfies the assumptions of Definition \ref{def:T_unconditional_1}.  Then the forecast $F$ for $Y$ is \emph{unconditionally}\/ $\myT'$-\emph{calibrated}\/ ($\UC_{\myT'}$) if 
\begin{align}  \label{eq:UC_T'}
\myE_{\hsp \myQ} \left[ V(\myT'(P_F) + \varepsilon, Y) \right] \geq 0 \quad \text{and} \quad 
\myE_{\hsp \myQ} \left[ V(\myT'(P_F) - \varepsilon, Y) \right] \leq 0 \quad \text{for all} \quad \varepsilon > 0.	
\end{align}
\end{definition}

\begin{proposition}  \label{prop:refinement}
Suppose that\/ $\myT'$ is a specification of a functional\/ $\myT$ that satisfies the assumptions of Definition \ref{def:T_unconditional_1}.  Then unconditional\/ $\myT$-calibration implies unconditional\/ $\myT'$-calibration.
\end{proposition}

\begin{proof}
As $\myT'(P_F) \in [\hsp \myT^-(P_F), \myT^+(P_F)]$, we note that \eqref{eq:UC_T} implies \eqref{eq:UC_T'}.
\end{proof}

We conclude this section by noting that various paths in the hierarchies from previous sections are included in the hierarchy displayed in Figure \ref{fig:hierarchy_T}.  For instance, in Figure \ref{fig:hierarchy_regression} the chain
\begin{align*}
\AC \Rightarrow \TC \Rightarrow \MC
\end{align*}
corresponds to the path $\AC \Rightarrow \CC_\myT \Rightarrow \UC_\myT$ in Figure \ref{fig:hierarchy_T} under the family of (non) exceedance probability functionals.  Similarly, with $\SQC$ denoting strong quantile calibration \eqref{eq:SQC}, the augmented chain
\begin{align*}
\AC \Rightarrow \SQC \Rightarrow \QC \Rightarrow \UQC
\end{align*}
corresponds to the path $\AC \Rightarrow \CC_\myT \Rightarrow \CC_{\myT'} \Rightarrow \UC_{\myT'}$ under the family of the full $\alpha$-quantile functionals, $\myT(P_F) = [ \hsp F^{-1}(\alpha), \hsp F^{-1}_+(\alpha)]$ in \eqref{eq:q_alpha}, and the corresponding family of specifications, $\myT'(P_F) = F^{-1}(\alpha) \in \myT(P_F)$, where $\alpha \in (0,1)$.

\section{Discussion}  \label{sec:discussion}

In this paper, we have reviewed, extended, and bridged concepts of calibration that have been proposed for classification and regression tasks, respectively, with emphasis on a rigorous study of hierarchical relations between the various notions.  While we have covered the most widely used notions of calibration in practice, we agree with \citet[p.~2]{Derr2025}, who note that  
\begin{quote}  \small
Our goal is \textit{not}\/ to provide an exhaustive list of notions of calibration. [\ldots] After trying this for some time, we gave up on this journey due to the immense variety and subtleties of the suggested variants of calibration.
\end{quote} 
Indeed, the literature on calibration has seen extensive growth recently, with decision calibration \citep{Sahoo2021, Zhao2021}, isotonic calibration \citep{Arnold2025}, tail calibration \citep{Allen2025b}, and instance-wise calibration \citep{Dey2025} being examples of concepts beyond the scope of this study.  However, for the concepts that have been most widely used in practice, our hierarchies are complete, except for the conjectured relations in Figures \ref{fig:hierarchy_regression} (a) and \ref{fig:hierarchy_classification_extended}.  

An in-depth understanding of notions of calibration and their relationships is vital in practice, to inform the development of displays, such as reliability diagrams \citep{Dimitriadis2021, Gneiting2023a}, for the diagnostic assessment of calibration, to guide the choice of calibration error metrics for the quantification of (mis)calibration \citep{Nixon2019, Gupta2022}, and to allow for statistical inference based on confidence or consistency sets, hypothesis tests, and related tools.  While these topical problems are beyond the scope of this paper, let us exemplify possible implications of our results.  In Section \ref{sec:classification}, we suggest a possible pragmatic approach for nominal outcomes, by checking for modal calibration in concert with marginal calibration and confidence calibration.  More in depth assessment of calibration may complement class-wise calibration checks with an assessment of probabilistic calibration across multiple (possibly random) permutations of the classes.  For discrete (ordinal) outcomes, the orthogonal concepts of probabilistic calibration and dual PIT calibration might complement each other.

Extensions of calibration diagnostics to multivariate (Euclidean) settings have been studied by \citet{Gneiting2008}, \citet{Ziegel2014}, \citet{Thorarinsdottir2016}, \citet{Knuppel2022}, and \citet{Allen2024}.  Such extensions from real-valued to multivariate settings face challenges similar to the transition from binary to multi-class settings stemming from high dimensionality and loss of total order.  As the key notion of probabilistic calibration is limited to the univariate case, dimension reduction techniques have commonly been applied.

The concepts of auto-calibration, marginal calibration, distribution calibration, and conditional calibration with respect to a functional or property apply to any type of outcome, even beyond outcomes in Euclidean spaces, which suggests far-reaching extensions of our study.  Future work might also elucidate the relationships between concepts of calibration, scoring rules, and score decompositions, as called for by \citet[Section II.C]{Huang2025} and addressed in part by \citet{Gruber2022}, \citet{Gneiting2023a}, and \citet{Arnold2024}.  A related task that we can only allude to here concerns methods for the adjustment of predictive distributions with the aim of achieving calibration.  Techniques of this type have been discussed in extraordinarily broad strands of literature that range from the statistical postprocessing of ensemble weather forecasts \citep{Gneiting2005b, Vannitsem2021} to modern methods of machine learning, as exemplified in recent work by \citet{HebertJohnson2018}, \citet{Jung2021}, \citet{Gupta2022}, \citet{Noarov2023}, and \citet{Dey2025}, among others.

\section*{Acknowledgments}

This work was initiated while L.\ Yang visited the Heidelberg Institute for Theoretical Studies (HITS) in 2025.  We thank Sam Allen, Sebastian Arnold, Rafael Frongillo, Alexander Henzi, Alexander Jordan, Sebastian Lerch, Jan St{\"u}hmer, and Johanna Ziegel for comments and discussion, and we are grateful for support by the Klaus Tschira Foundation.  J.\ Resin gratefully acknowledges support from the Deutsche Forschungs\-gemeinschaft (DFG, German Research Foundation) via project number 502572912.  L.\ Yang  gratefully acknowledges financial support from the National Science Foundation (DMS-2210712).

\bibliographystyle{apalike}
\bibliography{manuscript_arXiv}

\appendix

\renewcommand\thefigure{\thesection.\arabic{figure}}
\setcounter{figure}{0}

\renewcommand\thetable{\thesection.\arabic{table}}
\setcounter{table}{0}

\section{Additional proofs}  \label{app:proofs}

\begin{proof}[Proof of Proposition \ref{prop:implications_regression}]
\mbox{}
\begin{enumerate}
\item  Suppose \eqref{eq:CC} holds and let $\alpha \in (0,1)$.  Then,
\begin{align*}
\alpha \hsp = \hsp \myQ(Z_F \leq \alpha \mid F^{-1}(\alpha)) \hsp \geq \hsp \myQ(F(Y) \leq \alpha \mid F^{-1}(\alpha)) \hsp \geq \hsp \myQ(Y < F^{-1}(\alpha) \mid F^{-1}(\alpha)),
\end{align*}
as $y < F^{-1}(\alpha)$ implies $F(y) \leq \alpha$ for any fixed $y \in \real$ and any fixed cdf $F$.  Therefore, the first inequality in \eqref{eq:QC} holds.  

To demonstrate the second inequality in \eqref{eq:QC}, we first show that
\begin{align}  \label{eq:aux}
\alpha = \myQ(Z_F \leq \alpha \mid F^{-1}(\alpha)) = \myQ(Z_F < \alpha \mid F^{-1}(\alpha))
\end{align}
almost surely.  The first equality is immediate from \eqref{eq:CC}.  In order to prove the second equality, let us assume for a contradiction that $\myQ(A) > 0$, where
\begin{align*}
A = \left\{ \myQ \left( Z_F < \alpha \mid F^{-1}(\alpha) \right) < \myQ \left( Z_F \leq \alpha \mid F^{-1}(\alpha) \right) \right\}.
\end{align*}
As conditional (non) exceedance calibration implies probabilistic calibration, $\myQ(Z_F < \alpha) = \alpha$, whereas the law of total probability yields
\begin{align*}
\myQ(Z_F < \alpha) = \myQ(A) \hsp \myQ(Z_F < \alpha \mid A) + \myQ(\Omega \setminus A) \hsp \myQ(Z_F < \alpha \mid \Omega \setminus A) < \alpha,
\end{align*}
for the desired contradiction.  With \eqref{eq:aux}, we obtain
\begin{align*}
\alpha \hsp = \hsp \myQ \left( Z_F < \alpha \mid F^{-1}(\alpha) \right) \leq \hsp \myQ \left( F(Y-) < \alpha \mid F^{-1}(\alpha) \right) \leq \hsp \myQ \left( Y \leq  F^{-1}(\alpha) \mid F^{-1}(\alpha) \right) \! ,
\end{align*}
as $F(y-) < \alpha$ implies $y \leq  F^{-1}(\alpha)$ for any fixed $y \in \real$ and any fixed cdf $F$.
\item  Immediate by taking expectations.
\item  As noted by \citet{Gneiting2023a} in text following their Eq.~(10) and in their Appendix A.5, probabilistic calibration ensures that $\myQ(Y < F^{-1}_+(\alpha)) \leq \alpha \leq \myQ(Y \leq  F^{-1}(\alpha))$ for all $\alpha \in (0,1)$, which implies unconditional quantile calibration \eqref{eq:UQC}.
\item
The equivalence 
\begin{align*}
y <  F^{-1}(\alpha) \quad \Leftrightarrow \quad F(y) < \alpha
\end{align*}
for fixed $F$ and $y \in \real$ is immediate from the definition of the quantile function \eqref{eq:F_inverse} and implies the equality of the terms on the left-hand side of \eqref{eq:UQC} and \eqref{eq:WPC}, respectively. 
        
As $y \leq F^{-1}(\alpha)$ implies $F(y-) \leq \alpha$ for fixed $F$ and $y \in \real$, we have $\myQ(Y \leq F^{-1}(\alpha)) \leq \myQ(F(Y-) \leq \alpha)$.  Thus, unconditional quantile calibration \eqref{eq:UQC} implies weak probabilistic calibration \eqref{eq:WPC}.  For the reverse implication, we note that $F(y-) \leq \alpha - \varepsilon$ implies $y \leq  F^{-1}(\alpha)$ for all $\varepsilon > 0$ and fixed $F$ and $y \in \real$.  Therefore, \eqref{eq:WPC} implies
\begin{align*}
\myQ \! \left( Y \leq F^{-1}(\alpha) \right) \geq \hsp \lim_{\varepsilon \downarrow 0} \myQ \! \left( F(Y-) \leq \alpha - \varepsilon \right) \hsp \geq \hsp \lim_{\varepsilon \downarrow 0} \left( \alpha - \varepsilon \right) = \alpha,
\end{align*}
which is the second inequality in \eqref{eq:UQC}.
\end{enumerate}
\end{proof}

\begin{proof}[Proof of Proposition \ref{prop:binary}]
Letting $f(1) = P_F(\{ 1 \})$ and $g(1) = g(1 \mid F) = \myQ(Y = 1 \mid F) = \myQ(Y = 1 \mid f(1))$, we obtain $F(y) = f(1) \one(y \geq 1) + (1 - f(1)) \one(y \geq 2)$ and $G(y) = g(1) \one(y \geq 1) + (1 - g(1)) \one(y \geq 2)$.  Then $\alpha < f(1)$ implies $F_+^{-1}(\alpha) = 1$ and $F(F_+^{-1}(\alpha)-) = F(1-) = 0$.  Similarly, $\alpha \geq f(1)$ implies $F_+^{-1}(\alpha) = 2$ and $F(F_+^{-1}(\alpha)-) = F(2-) = f(1)$. Analogously, we obtain $G(F_+^{-1}(\alpha)-) = 0$ for $\alpha < f(1)$ and $G(F_+^{-1}(\alpha)-) = g(1)$ otherwise.  Therefore, we have
\begin{align*}
H(\alpha) = \myE_{\hsp \myQ} \! \left[ F(F_+^{-1}(\alpha)-) \right]  = \int_{(0,\alpha]} p \hsp\hsp \dd M(\hsp\hsp p)
\end{align*} 
and
\begin{align*}
K(\alpha) = \myE_{\hsp \myQ} \! \left[ \hsp G(F_+^{-1}(\alpha)-) \right] = \int_{(0,\alpha]} g(\hsp\hsp 1 \mid f(1) = p) \hsp \dd M(\hsp\hsp p)
\end{align*}
for $\alpha \in (0,1)$, where $M$ denotes the marginal law of the random variable $f(1)$ under $\myQ$.  We conclude that double PIT calibration implies
\begin{align*}
\int_{(0,\alpha]} \left( \, p - g(\hsp\hsp 1 \mid f(1) = p) \right) \hsp \dd M(\hsp\hsp p) = 0
\end{align*}	
for $\alpha \in (0,1)$, which yields auto-calibration with the arguments in the proof of the respective implication for probabilistic calibration in Theorem 2.11 of \citet{Gneiting2013}.  The converse is stated in Proposition \ref{prop:AC_DC_PC} (a). 
\end{proof}

\begin{proof}[Proof of Proposition \ref{prop:DC}]
Let $k \geq 2$, $\myE_{\hsp \myQ} \left[ F(k-1) \right] = \gamma$, and $\myQ(Y \leq k - 1) = \gamma'$.  Let the random variable $Y_F$ be such that $Y_F \mid F \sim F$, and let  
\begin{align*}
\widetilde{H}(\alpha) = \myQ \left( F(Y_F) \leq \alpha \mid Y_F \leq k - 1 \right)
\quad \text{and} \quad    
\widetilde{K}(\alpha) = \myQ \left( F(Y) \leq \alpha \mid Y \leq k - 1 \right)
\end{align*}
for $\alpha \in (0,1)$.  It then holds that 
\begin{align}  \label{eq:gamma_relations}
H(\alpha) & 
= \myQ \left( Y_F \leq k - 1 \right) \, \myQ \left( F(Y_F) \leq \alpha \mid Y_F \leq k - 1 \right)
= \gamma \, \widetilde{H}(\alpha)
\quad \text{and} \quad  
K(\alpha) = \gamma' \, \widetilde{K}(\alpha)
\end{align}
for $\alpha \in (0,1)$.  As the law of $F(Y)$ given $Y \leq k - 1$ is continuous, the arguments in the proof of Proposition \ref{prop:AC_DC_PC} \ref{prop:DC_uniform} imply that
\begin{align}  \label{eq:proof_argument}
\widetilde{H}(\alpha) = \widetilde{K}(\alpha) \text{ for } \alpha \in (0,1) \text{ if, and only if, the law of } \widetilde{H}(F(Y)) \text{ given } Y \leq k - 1 \text{ is standard uniform.} 
\end{align}
\vspace{-1.2\baselineskip} \\
If $F$ is double PIT calibrated then
\begin{align*}
\textstyle
\gamma' = \lim_{\alpha \uparrow 1} K(\alpha) = \lim_{\alpha \uparrow 1} H(\alpha) = \gamma \, \myQ(F(k-1) < 1) = \gamma,
\end{align*}
where the first equality is implied by the continuity of $\widetilde{K}$, the second by double PIT calibration, and the last by the assumption that $\myQ(F(k-1) < 1) = 1$.  By \eqref{eq:gamma_relations}, $\widetilde{H}(\alpha) = \widetilde{K}(\alpha)$ for $\alpha \in (0,1)$, and by the equivalence \eqref{eq:proof_argument}, $\widetilde{H}(F(Y))$ given $Y \leq k - 1$ is uniform on $(0,1)$, which in view of \eqref{eq:gamma_relations} is equivalent to $H(F(Y))$ being uniform on $(0,\gamma)$ conditional on  $Y \leq k - 1$.

Conversely, suppose that $\gamma = \gamma'$ and let the conditional law of $H(F(Y))$ given $Y \leq k - 1$ be uniform on $(0,\gamma)$.  By \eqref{eq:gamma_relations}, $\widetilde{H}(F(Y))$ given $Y \leq k - 1$ is uniform on $(0,1)$, and by the equivalence \eqref{eq:proof_argument}, $\widetilde{H}(\alpha) = \widetilde{K}(\alpha)$ for $\alpha \in (0,1)$.  Again by \eqref{eq:gamma_relations}, $H(\alpha) = K(\alpha)$ for $\alpha \in (0,\gamma)$. 
\end{proof}

\begin{proof}[Proof of claims in Example \ref{ex:nDC}]
Let $W \sim \operatorname{Unif}(0,\frac{1}{2})$, and $Y \mid W \sim \operatorname{Bernoulli}(1-2W)$.  The predictive cdf is $F(y) = W \one(y \geq 1) + (1-W) \one(y \geq 2)$, whereas the auto-calibrated cdf is $G(y) = 2W \allowbreak\one(y \geq 1) + (1-2W) \one(y \geq 2)$.  Then $F(Y) = W \one (Y = 1) + \one (Y = 2)$.  Let $\alpha \in (0,1)$.  We have $F^{-1}_+(\alpha) = 1 + \one(\alpha \geq W)$, and it follows that
\begin{align*}
F \left( F^{-1}_+(\alpha)- \right) = \begin{cases} 0, & \text{if } \alpha < W, \\ W, & \text{otherwise}, \end{cases}
\qquad
G \left( F^{-1}_+(\alpha)- \right) = \begin{cases} 0, & \text{if }\alpha < W, \\ 2W, & \text{otherwise}. \end{cases}
\end{align*}
Then 
\begin{align*}
H(\alpha) = \myE_{\hsp \myQ} \left( F(F^{-1}_+(\alpha)-)\right) = 
\begin{cases}
\alpha^2,    & \text{if } \alpha \leq \frac{1}{2}, \\
\frac{1}{4}, & \text{otherwise},
\end{cases}
\end{align*}
whereas
\begin{align*}
K(\alpha) = \myQ \left( F(Y) \leq \alpha \right) = \myE_{\myQ} \left( G(F^{-1}(\alpha)-) \right) = 
\begin{cases}
2 \alpha^2,  & \text{if }\alpha \leq \frac{1}{2}, \\
\frac{1}{2}, & \text{otherwise}.
\end{cases}
\end{align*}
Thus $K \neq H$, despite $\myQ(H(F(Y)) \leq \alpha \mid Y < 2) = 4 \alpha$ for $0 \leq \alpha \leq \frac{1}{4} = \myE_{\myQ} [ F(1)]$.
\end{proof}

\begin{proof}[Proof of Proposition \ref{prop:permutation}]  With $\rho \in S_k$ denoting the permutation that inverts the order of the classes, we have
\begin{align*}
F_\rho(k+1-t) = P_{F_\rho}(\{ y : y \leq k + 1 - t \}) = P_F(\{ y : y \geq t \}) = 1 - P_F(\{ y : y < t \}) = 1 - F(t-)
\end{align*}
for $t \in \real$ and
\begin{align*}
k + 1 - F_\rho^{-1}(\alpha) & = k + 1 - \inf \{ y : F_\rho(y) \geq \alpha \} = \sup \{ k + 1 - (k + 1 - y) : F_\rho(k + 1 - y) \geq \alpha \} \\
& = \sup \{ y : F(y-) \leq 1 - \alpha \} = \sup \{ y : F(y) \leq 1 - \alpha \} = F_+^{-1}(1 - \alpha)
\end{align*}
for $\alpha \in (0,1)$.  Furthermore, $F(y-) = F(y - 1)$ and $F(y) = 1 - F_\rho(k - y)$ for $y \in \{ 1, \ldots, k \}$ and, therefore,
\begin{align*}
Z_F & = F(Y-1) + V(F(Y) - F(Y-1)) \\
	& = 1 - F_\rho(k - Y + 1) + V \! \left(1 - F_\rho(k - Y) - \left( 1 - F_\rho(k - Y + 1) \right) \right) \\
	& = 1 - F_\rho(Y_\rho) - V \! F_\rho(Y_\rho - 1) + V F_\rho(Y_\rho) + [F_\rho(Y_\rho - 1) - F_\rho(Y_\rho - 1)] \\
    & =  1 - \left[F_\rho(Y_\rho - 1) + (1 - V)(F_\rho(Y_\rho) - F_\rho(Y_\rho - 1)\right] \\
    & \overset{d}{=} 1 - Z_{F_\rho}
\end{align*}
where the last equality is in distribution. 

\medskip

\noindent
Case $\NC = \PC$: As a direct consequence of the preceding observation, $F$ is probabilistically calibrated for $Y$ if, and only if, $F_\rho$ is probabilistically calibrated for $Y_\rho$.

\medskip

\noindent
Case $\NC = \CC$:  Let $\alpha \in (0,1)$ and $q \in \real$.  Then
\begin{align}  \label{eq:proof1}
\myQ \left( Z_{F_\rho} \leq \alpha \mid F_{\rho}^{-1}(\alpha) \geq q \right) 
& = 1 - \myQ \left( Z_F < 1 - \alpha \mid F_+^{-1}(1 - \alpha) \leq k + 1 - q \right) \nonumber \\ 
& = 1 - {\textstyle \lim_{\varepsilon \downarrow 0}} \: \myQ \left( Z_F < 1 - \alpha \mid F^{-1}(1 - \alpha + \varepsilon) \leq k + 1 - q \right) \nonumber \\
& \geq 1 - {\textstyle \lim_{\varepsilon \downarrow 0}} \: \myQ \left( Z_F \leq 1 - \alpha + \varepsilon \mid F^{-1}(1 - \alpha + \varepsilon) \leq k + 1 - q \right) \nonumber \\
& = 1 - {\textstyle \lim_{\varepsilon \downarrow 0}} \left( 1 - \alpha + \varepsilon \right) = \alpha
\end{align}
almost surely, where we obtain the second equality as follows.  Let $A = \{ Z_F \leq 1 - \alpha \}, B = \allowbreak \{ F_+^{-1}(1 - \alpha) \leq k + 1 - q \}$, and $B_\varepsilon = \{ F^{-1}(1 - \alpha + \varepsilon) \leq k + 1 - q \}$. As 
\begin{align*}
F_+^{-1}(1 - \alpha) \leq F^{-1}(1 - \alpha + \varepsilon) \leq F_+^{-1}(1 - \alpha + \varepsilon) \overset{\varepsilon \downarrow 0}{\longrightarrow} F_+^{-1}(1 - \alpha)
\end{align*}
by the right-continuity of $F_+^{-1}$, we have $\lim_{\varepsilon\downarrow 0} F^{-1}(1 - \alpha + \varepsilon) = F_+^{-1}(1 - \alpha)$.  For each $\omega \in \Omega$ there exists $\varepsilon_0(\omega)$ such that $F(\omega)^{-1}(1-\alpha + \varepsilon) < F(\omega)^{-1}_+(1-\alpha) + 1$ for all $\varepsilon < \varepsilon_0(\omega)$, and thus $F(\omega)^{-1}(1-\alpha + \varepsilon) = F(\omega)^{-1}_+(1-\alpha)$ for all $\varepsilon < \varepsilon_0(\omega)$ because of the discreteness of the quantile functions, which only take integer values.  Therefore the indicator $I_{B_\varepsilon}(\omega) = \one(\omega \in B_\varepsilon)$ converges pointwise to $I_B(\omega) = \one(\omega \in B)$.  With $I_A(\omega) = \one(\omega \in A)$, we obtain the desired equality 
\begin{align*}
\myQ(A\mid B_\varepsilon) = \frac{\myE_{\myQ}[I_AI_{B_\varepsilon}]}{\myE_{\myQ}[I_{B_\varepsilon}]} \overset{\varepsilon\downarrow 0}{\longrightarrow} \frac{\myE_{\myQ}[I_AI_{B}]}{\myE_{\myQ}[I_{B}]} = \myQ(A\mid B)
\end{align*}
as long as $\myQ(B) > 0$, where we have used the dominated convergence theorem.

In an analogous fashion, we obtain
\begin{align}  \label{eq:proof2}
\myQ \left( Z_{F_\rho} \leq \alpha \mid F_{\rho}^{-1}(\alpha) < q \right) 
& = 1 - \myQ \left( Z_F < 1 - \alpha \mid F_+^{-1}(1 - \alpha) > k + 1 - q \right) \geq \alpha.
\end{align}
As conditional (non) exceedance probability implies probabilistic calibration, and probabilistic calibration of $F$ for $Y$ implies probabilistic calibration of $F_\rho$ for $Y_\rho$ by the previous case, the law of total probability yields
\begin{align*}
\alpha = \myQ \left( Z_{F_\rho} \leq \alpha \right)
& = \myQ \left( F_{\rho}^{-1}(\alpha) < q \right) \myQ \left( Z_{F_\rho} \leq \alpha \mid F_{\rho}^{-1}(\alpha) < q \right) \\
& \mbox{} \qquad + \, \myQ \left( F_{\rho}^{-1}(\alpha) \geq q \right) \myQ \left( Z_{F_\rho} \leq \alpha \mid F_{\rho}^{-1}(\alpha) \geq q \right) .
\end{align*}
Owing to \eqref{eq:proof2} we deduce that 
\begin{align}  \label{eq:proof3}
\myQ(Z_{F_\rho} \leq \alpha \mid F_{\rho}^{-1}(\alpha) \geq q) = \frac{\alpha - (1 - \myQ(F_{\rho}^{-1}(\alpha) \geq q)) \: \myQ(Z_{F_\rho} \leq \alpha \mid F_{\rho}^{-1}(\alpha) < q)}{\myQ(F_{\rho}^{-1}(\alpha) \geq q)} \leq \alpha
\end{align}
almost surely whenever $\myQ(F_{\rho}^{-1}(\alpha) \geq q) > 0$.  As the family of events $A_q = \{ F_{\rho}^{-1}(\alpha) \geq q \}$ is closed under intersections and generates the $\sigma$-algebra $\sigma(F_{\rho}^{-1}(\alpha))$ generated by $F_{\rho}^{-1}(\alpha)$, Theorem 33.1 of \citet{Billingsley2013} applies, and we conclude from the inequalities \eqref{eq:proof1} and \eqref{eq:proof3} that $\myQ(Z_{F_\rho} \leq \alpha \mid F_{\rho}^{-1}(\alpha)) = \alpha$ almost surely.

\medskip

\noindent
Case $\NC = \TC$:  If $t \in \real \setminus \{ 1, \ldots, k \}$ then $1 - F_\rho(k + 1 - t) = F(t-) = F(t)$ and $\myQ(Y < t) = \myQ(Y \leq t)$ and, therefore, 
\begin{align*}
\myQ \left( Y_\rho \leq k + 1 - t \mid F_\rho(k+1 - t) \right) & = 1 - \myQ \left( Y < t \mid F(t) \right) \\ 
& = 1 - \myQ \left( Y \leq t \mid F(t) \right) = 1 - F(t) = F_\rho(k + 1 -t).
\end{align*}
If $t \in \{ 1, \ldots, k \}$ then $F_\rho(k + 1 - t) = 1 - F(t - 1)$ and $\myQ(Y < t) = \myQ(Y \leq t - 1)$ and, therefore,
\begin{align*}
\myQ \left( Y_\rho \leq k + 1 - t \mid F_\rho(k + 1 - t) \right) & = 1 - \myQ \left( Y < t \mid F(t - 1) \right) \\ 
& = 1 - \myQ \left( Y \leq t - 1 \mid F(t - 1) \right) = 1 - F(t - 1) = F_\rho(k + 1 - t).
\end{align*}
This establishes threshold calibration of $F_\rho$ for $Y_\rho$, as we have shown that $\myQ( Y_\rho \leq u \mid F_\rho(u) ) = F_\rho(u)$ for all $u \in \real$.
\end{proof}

\begin{proof}[Proof of Proposition \ref{prop:mixLS}.]
Let $\Lambda$ be a probability measure on $\cP$ with $\Lambda(\LS_\myT(s)) = 1$ for some $s \subseteq \real$, and let $Q = \int P \dd \Lambda(P)$.  Let $x \in \cT$ and $t \in s$.  Following the proof of Proposition 3.4 in \citet{Fissler2021}, we note that
\begin{align*}
\myE_{Y \sim P}[\myS(x,Y) - \myS(t,Y)] 
\begin{cases}
= 0, & \text{if } x \in s, \\
> 0, & \text{otherwise},
\end{cases}
\end{align*}
for $P \in \LS_\myT(s)$.  Therefore,
\begin{align*}
\myE_{Y\sim Q}[\myS(x,Y) - \myS(t,Y)]
& = \int [\myS(x,y) - \myS(t,y)] \dd Q(y) \\
& = \int \int [\myS(x,y) - \myS(t,y)] \dd P(y) \dd \Lambda(P) \\ 
& = \int \myE_{Y \sim P}[\myS(x,Y) - \myS(t,Y)] \dd \Lambda(P) 
\begin{cases}
= 0, & \text{if } x \in s, \\ 
> 0, & \text{otherwise}.
\end{cases}
\end{align*}
Thus, the expected score function $\myE_{Y \sim Q} [\myS(x,Y)]$ attains its minimum if, and only if, $x \in s$, which demonstrates that the functional is also elicitable under $Q$ with $\myT(Q) = s$.
\end{proof}

\section{Proofs of the claims in the tables}  \label{app:examples}

In this appendix, we establish the properties stated in Tables \ref{tab:examples_classification}, \ref{tab:examples_regression}, \ref{tab:examples_ordinal}, and \ref{tab:examples_classification_extended}.  The following claims are easy to establish with paper and pencil, and we leave the straightforward arguments to the reader.
\begin{itemize} 
\item
All claims in Table \ref{tab:examples_classification}.
\item 
The claims for the examples from \citet{Gneiting2023a} in Table \ref{tab:examples_regression}.  While we use a weaker notion of quantile calibration in this paper than in \citet{Gneiting2023a}, these examples are not affected, as they feature continuous distributions with unique quantiles.
\item 
The claims for Example \ref{ex:SRS} in Tables \ref{tab:examples_regression}, \ref{tab:examples_ordinal}, and \ref{tab:examples_classification_extended}.
\item 
All claims for auto-calibration ($\AC$), class-wise confidence calibration ($\CwCoC$), confidence calibration ($\CoC$), and modal calibration ($\CMC$). 
\end{itemize}

Generally, in view of the hierarchical relations summarized in Figures \ref{fig:hierarchy_classification}, \ref{fig:hierarchy_regression}, and \ref{fig:hierarchy_classification_extended}, it suffices to establish a proper subset of the claims in Tables \ref{tab:examples_regression}, \ref{tab:examples_ordinal}, and \ref{tab:examples_classification_extended}.  Table \ref{tab:app_examples} collects the properties to be demonstrated, using the acronyms from the figures, where we overline the acronym if it is to be shown that a notion does not hold.  

\begin{table}[t]
\centering
\caption{Notions of calibration to be verified ($\NC$) or falsified ($\overline{\NC}$) in order to establish the claims in Tables \ref{tab:examples_regression}, \ref{tab:examples_ordinal}, and \ref{tab:examples_classification_extended}, up to the straightforward cases noted in the text.  Acronyms used are from Definitions \ref{def:AC_MC}, \ref{def:classification}, \ref{def:regression}, \ref{def:DC}, and \ref{def:classification_extended}.  \label{tab:app_examples}}
\begin{tabular}{lll}
\toprule
Table & Example & To be Demonstrated \\
\toprule
Table \ref{tab:examples_regression} & Example \ref{ex:dPC} & $\QC$, $\overline{\PC}$, $\overline{\MC}$ \\
Table \ref{tab:examples_regression} & Example \ref{ex:TC+nUQC} & $\TC$, $\overline{\UQC}$ \\
\midrule
Tables \ref{tab:examples_ordinal} and \ref{tab:examples_classification_extended} & Example \ref{ex:dnwPC} & $\overline{\MC}$, $\overline{\UQC}$, $\DC$; $\overline{\pPC}$, $\overline{\aPC}$ \\ 
Tables \ref{tab:examples_ordinal} and \ref{tab:examples_classification_extended} & Example \ref{ex:cond} & $\overline{\DC}$; $\fCC$, $\fTC$ \\
\midrule
Table \ref{tab:examples_classification_extended} & \citet[Table 1]{Vaicenavicius2019} & $\pCC$, $\overline{\fPC}$, $\aPC$, $\fTC$ \\
Table \ref{tab:examples_classification_extended} & \citet[Supplement, Table 2]{Vaicenavicius2019} &  $\overline{\CwC}$, $\MC$, $\overline{\pPC}$, $\overline{\aPC}$, $\overline{\pTC}$ \\
Table \ref{tab:examples_classification_extended} & \citet[Footnote 2]{Silva2023} & $\pCC$, $\overline{\aPC}$, $\fTC$ \\
Table \ref{tab:examples_classification_extended} & \citet[Example 2.4 (b)]{Gneiting2023a} & $\overline{\CwC}$, $\MC$, $\pCC$, $\overline{\aPC}$, $\pTC$ \\
Table \ref{tab:examples_classification_extended} & \citet[Example 2.14 (b)]{Gneiting2023a} & $\overline{\pPC}$, $\aPC$, $\fTC$ \\
Table \ref{tab:examples_classification_extended} & Example \ref{ex:CwC} & $\CwC$, $\overline{\pPC}$, $\aPC$, $\overline{\pTC}$ \\
Table \ref{tab:examples_classification_extended} & Example \ref{ex:pPC+nMC} & $\overline{\MC}$, $\overline{\pCC}$, $\overline{\fPC}$, $\pPC$, $\aPC$ \\
Table \ref{tab:examples_classification_extended} & Example \ref{ex:fPC+nCoC} & $\overline{\CwC}$, $\MC$, $\overline{\pCC}$, $ \fPC$, $\overline{\pTC}$ \\
Table \ref{tab:examples_classification_extended} & Example \ref{ex:fCC+nCoC} & $\fCC$, $\fTC$ \\
\bottomrule
\end{tabular}
\end{table}

We now explain how we ensure calibration properties by construction, taking Example 2.4 (b) from \citet{Gneiting2023a} as a blue print.  In this example, we have $k = 3$ classes and $N = 3$ equiprobable predictive distributions, which we prescribe as follows.

\begin{center}
\begin{tabular}{>{$}c<{$} | >{$}c<{$} | >{$}c<{$} >{$}c<{$} >{$}c<{$}}
\toprule
j & \myQ(F = F_j) & f_j(1) & f_j(2) & f_j(3) \\
\midrule
1 & 1/3 & 1/2 & 1/4 & 1/4 \\ 
2 & 1/3 & 1/4 & 1/2 & 1/4 \\ 
3 & 1/3 & 1/4 & 1/4 & 1/2 \\ 
\bottomrule
\end{tabular}
\end{center}

We follow the flow of ideas in the proof of Proposition \ref{prop:constructive} in constructing the linear constraints for marginal calibration and probabilistic calibration in the general setting, and then we return to our example and showcase the particular linear system used in its construction.  In the general setting, we assume forecast probabilities of the from $m/n$ for positive integers $n \geq k$ and $m \in \{ 1, \ldots, n - k + 1 \}$.  We use the shorthand $p_j = \myQ(F = F_j)$ to denote the probability mass of forecast $F_j$.  The goal is to find conditional outcome probabilities $g_j(i) = \myQ(Y = i \mid F = F_j)$, where $i = 1, \ldots, k$ and $j = 1, \ldots, N$, such that the random forecast $F$ is both marginally calibrated and probabilistically calibrated. 
To begin with, we ensure that the conditional outcome probabilities sum to one by invoking the linear constraint
\begin{align}  \label{eq:linProbSumtoOne}
\sum_{i = 1}^k g_j(i) = 1
\end{align}
for $j = 1, \ldots, N$.

Marginal calibration corresponds to the expected forecast probability matching the marginal class probability 
for each class, which yields the linear constraint
\begin{align}  \label{eq:linMC}
\sum_{j = 1}^N p_j \hsp \hsp g_j(i) = \sum_{j = 1}^N p_j \hsp f_j(i)
\end{align}
for $i = 1, \ldots, k$.  To formalize probabilistic calibration, we consider the conditional distribution of the randomized PIT $Z_F$ given $F = F_j$, where $j = 1, \ldots, N$.  This law is described by a piecewise linear cdf with $k - 1$ possible breakpoints at $F_j(i)$, where $i = 1, \ldots, k - 1$, namely,  
\begin{align}  \label{eq:condCDFofPIT}
\begin{multlined}[t]
u \longmapsto \myQ(Z_F \leq u \mid F = F_j) = \sum_{i=1}^k q_{i,j}(u) \, g_j(i) \\
\text{with} \qquad q_{i,j}(u) = \myQ(Z_F \leq u \mid Y = i, F = F_j) = 
\begin{cases}
0,                         & \text{if } u < F_j(i-1), \\
\frac{u-F_j(i-1)}{f_j(i)}, & \text{if } F_j(i-1) \leq u < F_j(i), \\
1,                         & \text{if } u \geq F_j(i).
\end{cases}
\end{multlined}
\end{align}
Note that $q_{i,j}$ is a conditional version of the nonrandomized PIT from \citet[][Eq.\ (2)]{Czado2009}, which is an alternative to the randomized PIT for count data.  As $F$ itself follows a discrete distribution, the (unconditional) cdf of $Z_F$ is simply the (probability-weighted) pointwise average of the conditional cdfs of $Z_F$, that is,
\begin{align}  \label{eq:cdf_PIT}
u \longmapsto \myQ(Z_F \leq u) = \sum_{j = 1}^N p_j \hsp \hsp \myQ(Z_F \leq u \mid F = F_j) = \sum_{i,j} p_j \hsp \hsp q_{i,j}(u) \hsp \hsp g_j(i),
\end{align}
which is a piecewise linear function with possible breakpoints at $u = F_j(i)$ for $i = 1, \ldots, k - 1$ and $j = 1, \ldots, N$.  The forecast $F$ is probabilistically calibrated if the cdf in \eqref{eq:cdf_PIT} is the identity function on the unit interval, which is equivalent to the breakpoints being preserved.  Thus, we obtain the linear constraint
\begin{align}  \label{eq:linPC}
\sum_{i,j} p_j \hsp \hsp q_{i,j}(u) \hsp \hsp g_j(i) = u
\end{align}
at $u = 1/n, \ldots, (n-1)/n$.  

In the above case of Example 2.4 (b) from \citet{Gneiting2023a}, where we have $k = 3$ classes, $N = 3$ distinct predictive distributions, and $n = 4$, the constraints from \eqref{eq:linProbSumtoOne}, \eqref{eq:linMC}, and \eqref{eq:linPC} yield the linear system
\begin{align*}
\left(\begin{array}{ccccccccc}
    1 & 0 & 0 & 1 & 0 & 0 & 1 & 0 & 0 \\ 
    0 & 1 & 0 & 0 & 1 & 0 & 0 & 1 & 0 \\ 
    0 & 0 & 1 & 0 & 0 & 1 & 0 & 0 & 1 \\ 
    1/3 & 1/3 & 1/3 & 0 & 0 & 0 & 0 & 0 & 0 \\ 
    0 & 0 & 0 & 1/3 & 1/3 & 1/3 & 0 & 0 & 0 \\ 
    0 & 0 & 0 & 0 & 0 & 0 & 1/3 & 1/3 & 1/3 \\ 
    1/6 & 1/3 & 1/3 & 0 & 0 & 0 & 0 & 0 & 0 \\ 
    1/3 & 1/3 & 1/3 & 0 & 1/6 & 1/3 & 0 & 0 & 0 \\ 
    1/3 & 1/3 & 1/3 & 1/3 & 1/3 & 1/3 & 0 & 0 & 1/6 \\ 
\end{array} \right)
\begin{pmatrix}
    g_1(1) \\
    g_1(2) \\
    g_1(3) \\
    g_2(1) \\
    g_2(2) \\
    g_2(3) \\
    g_3(1) \\
    g_3(2) \\
    g_3(3) \\
\end{pmatrix}
= \begin{pmatrix}
    1 \\
    1 \\
    1 \\
    1/3 \\
    1/3 \\
    1/3 \\
    1/4 \\
    2/4 \\
    3/4 \\
\end{pmatrix},
\end{align*}
which has the particular solution $g^{(\hsp p)} = (f_1(1), f_1(2), f_1(3), f_2(1), f_2(2), f_2(3), f_3(1), f_3(2), f_3(3))^T$ and the homogeneous solution $g^{(h)} = (0, 1, {-1}, 1, {-2}, 1, {-1}, 1, 0)^T$, and thus general solutions of the form $g^{(\hsp p)} + c\cdot g^{(h)}$.  By choosing $c$ sufficiently close to zero, we ensure that the posited conditional outcome probabilities are nonnegative.  In particular, the choice $c = - 3/20$ recovers Example 2.4 (b) from \citet{Gneiting2023a}, as follows.
\begin{center}
\begin{tabular}{>{$}c<{$} | >{$}c<{$} | >{$}c<{$} >{$}c<{$} >{$}c<{$} | >{$}c<{$} >{$}c<{$} >{$}c<{$}}
\toprule
j & \myQ(F = F_j) & f_j(1) & f_j(2) & f_j(3) & g_j(1) & g_j(2) & g_j(3) \\
\midrule
1 & 1/3 & 1/2 & 1/4 & 1/4 & 5/10 & 1/10 & 4/10 \\ 
2 & 1/3 & 1/4 & 1/2 & 1/4 & 1/10 & 8/10 & 1/10 \\ 
3 & 1/3 & 1/4 & 1/4 & 1/2 & 4/10 & 1/10 & 5/10 \\ 
\bottomrule
\end{tabular}
\end{center}

Beyond marginal calibration and probabilistic calibration, many other notions of calibration can be encoded in terms of linear constraints on the conditional outcome probabilities.  We leverage this approach in the supplemental code in order to construct examples with particular properties.  In what follows, we collect the respective constraints, as summarized in Table \ref{tab:linNC} for notions of calibration from Definitions \ref{def:AC_MC}, \ref{def:classification}, \ref{def:regression}, and \ref{def:DC}.  The full notions from Definition \ref{def:classification_extended} are covered by applying the constraints to all permutations of the classes.  For average probabilistic calibration, we average the constraints from full probabilistic calibration across the permutations.

\begin{table}[t]
\centering
\caption{Notions of calibration that we encode in the form of linear constraints on the conditional outcome probabilities, in the order of appearance in, and using acronyms from, Definitions \ref{def:AC_MC}, \ref{def:classification}, \ref{def:regression}, and \ref{def:DC}.  \label{tab:linNC}}
\begin{tabular}{l|cccccc}
\toprule
Notion & $\MC$ & $\CwC$ & $\PC$ & $\CC$ & $\TC$ & $\DC$ \\ 
\midrule
Linear constraints & \eqref{eq:linMC} & \eqref{eq:linCwC} & \eqref{eq:linPC} & \eqref{eq:linCC} & \eqref{eq:linTC} & \eqref{eq:linDC} \\ 
\bottomrule
\end{tabular}
\end{table}

Conditional (non) exceedance probability calibration ($\CC$) splits the linear constraint in \eqref{eq:linPC} into multiple constraints conditional on the predicted $u$-quantiles $F^{-1}(u)$.  For each $u\in (0,1)$, the conditioning gives rise to a partition of the forecast cases $1, \ldots, N$ into subsets $P^{\CC}(u,y) = \{ \hsp j : F_j^{-1}(u) = y \}$, where $y \in \{ 1, \ldots, k \}$.  The subsets $P^{\CC}(u,y)$ are constant in $u$ on the half-open interval $((m-1)/n,m/n]$ for any $m = 1, \ldots, n$, by construction of the probability mass functions $f_j$ and the left-continuity of the (lower) quantile function.  Owing to the piecewise linear nature of the conditional cdf at \eqref{eq:condCDFofPIT}, it suffices to ensure the $\CC$ condition \eqref{eq:CC} for two distinct values from these intervals, such as $u = (2m-1)/(2n)$ and $u = m/n$.  With 
\begin{align*}
\overline{p}_j^{\hsp \CC}(u,y) = \begin{cases}
p_j / \left( \sum_{r \in P^{\CC}(u,y)} p_r \right), & \text{if } j \in P^{\CC}(u,y), \\
0,                                                  & \text{otherwise},
\end{cases}
\end{align*}
the condition reduces to the linear constraint
\begin{align}  \label{eq:linCC}
\sum_{i,j} \overline{p}_j^{\hsp \CC}(u,y) \hsp \hsp q_{i,j}(u) \hsp \hsp g_j(i) = u
\end{align}
for $u = 1/(2n), 2/(2n), \ldots, (2n-1)/(2n)$ and $y = 1, \ldots, k$ such that $P^{\CC}(u,y)$ is non-empty.

\begin{figure}[t]
\centering
\includegraphics[width=0.49\linewidth]{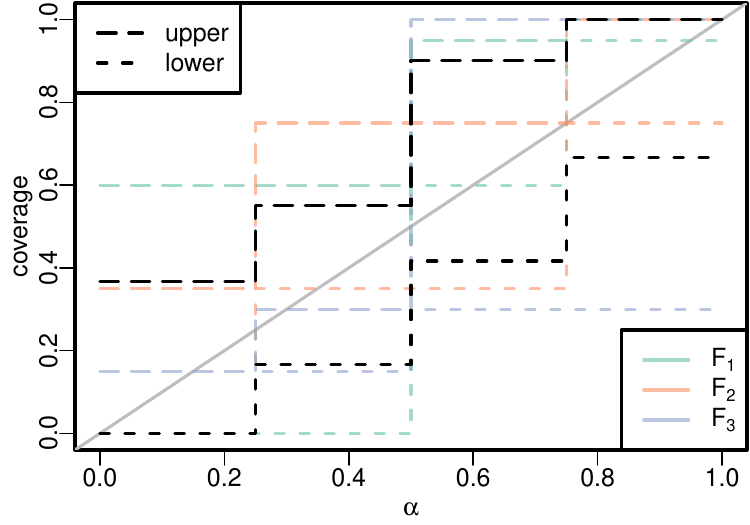}
\includegraphics[width=0.49\linewidth]{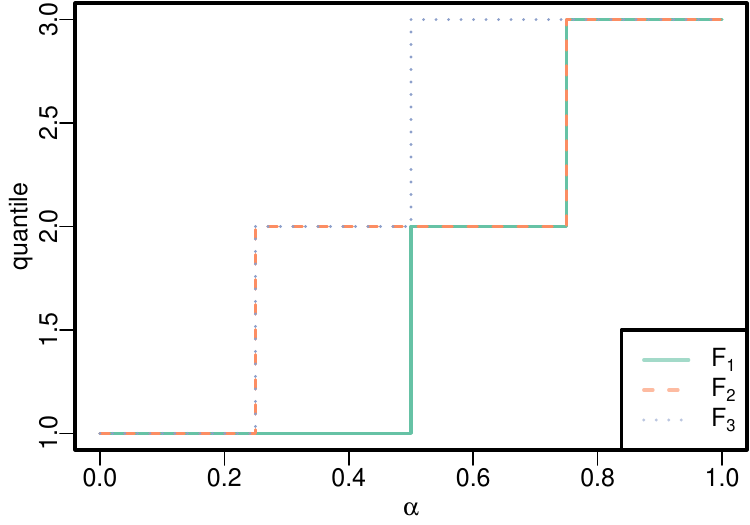}
\caption{Conditional and unconditional quantile coverage (left) and predictive quantile functions (right) in Example \ref{ex:dPC}.  The colored lines in the left panel show conditional coverage, the black lines show unconditional lower coverage and upper coverage, respectively.  \label{fig:cond_cov}}
\end{figure}

In a similar fashion, threshold calibration ($\TC$) reduces to linear constraints on the conditional outcome probabilities.  Here, it suffices to consider thresholds $t = 1, \ldots, k - 1$ at which the predicted cdfs have jumps and do not all attain the value one.  Each such threshold $t$ yields a partition into subsets $P^{\TC}(u,t) = \{ j : F_j(t) = u \}$, where $u = 1/n, \ldots, (n-1)/n$.  With
\begin{align*}
\overline{p}_j^{\hsp \TC}(u,t) = 
\begin{cases}
p_j / (\sum_{r \in P^{\TC}(u,t)} p_r), & \text{if } j \in P^{\TC}(u,t), \\
0,                                     & \text{otherwise},
\end{cases}
\end{align*}
we obtain the linear constraint
\begin{align}  \label{eq:linTC}
\sum_{i,j} \overline{p}_j^{\hsp \TC}(u,t) \hsp \hsp \one(i \leq t) \, g_j(i) = u
\end{align}
for $u = 1/n, \ldots, (n-1)/n$ and $t = 1, \ldots, k-1$ such that $P^{\TC}(u,t)$ is non-empty.

To formulate double PIT calibration ($\DC$) in terms of linear constraints, we note that the cdf of the simplified PIT $F(Y)$ is given by $K(u) = \sum_{i,j} p_j \, \widetilde{q}_{i,j}(u) \, g_j(i)$ with $\widetilde{q}_{i,j}(u) = \one(F_j(i) \leq u)$, which is obtained by replacing $q_{i,j}$ with $\widetilde{q}_{i,j}$ in \eqref{eq:condCDFofPIT} and \eqref{eq:cdf_PIT}, whereas the hypothesized cdf is given by $H(u) = \sum_{i,j} p_j \, \widetilde{q}_{i,j}(u) \hsp \hsp f_j(i)$.  In the setting of our examples, the functions $H$ and $K$ are constant on intervals of the form $[(m-1)/n,m/n)$, and thus the $\DC$ condition $K = H$ reduces to
\begin{align}  \label{eq:linDC}
\sum_{i,j} p_j \hsp \hsp \widetilde{q}_{i,j}(u) \hsp \hsp g_j(i) = \sum_{i,j} p_j \hsp \hsp \widetilde{q}_{i,j}(u) \hsp \hsp f_j(i)
\end{align}
for $u = 1/n, \ldots, (n-1)/n$.

To address class-wise calibration ($\CwC$), we refine the linear constraints in \eqref{eq:linMC} by conditioning on the predicted class probability $f_j(i)$.  For $i = 1, \ldots, k$, the conditioning gives rise to a partition of the forecast cases $1, \ldots, N$ into subsets $P^{\hsp \CwC}(u,i) = \{ j : f_j(i) = u \}$, where $u \in (0,1)$.  Letting 
\begin{align*}
\overline{p}_j^{\hsp \CwC}(u,i) = 
\begin{cases}
p_j / (\sum_{r \in P^{\CwC}(u,i)} p_r), & \text{if } j \in P^{\CwC}(u,i), \\
0,                                      & \text{otherwise},
\end{cases}
\end{align*}
we obtain the linear constraint
\begin{align}  \label{eq:linCwC}
\sum_j \overline{p}^{\hsp \CwC}_j(u,i) \hsp \hsp g_j(i) = u
\end{align}
for $i = 1, \ldots, k$ and $u = 1/n, \ldots, (n+1-k)/n$ such that $P^{\hsp \CwC}(u,i)$ is non-empty.

The supplementary code provides implementations in \textsf{R} \citep{R} of the example generation and the above linear conditions, which we use to verify or falsify the notions in Table \ref{tab:linNC} and the respective versions from Definition \ref{def:classification_extended} as listed in Table \ref{tab:app_examples}.

\begin{figure}[t]
\centering
\includegraphics[width=0.49\linewidth]{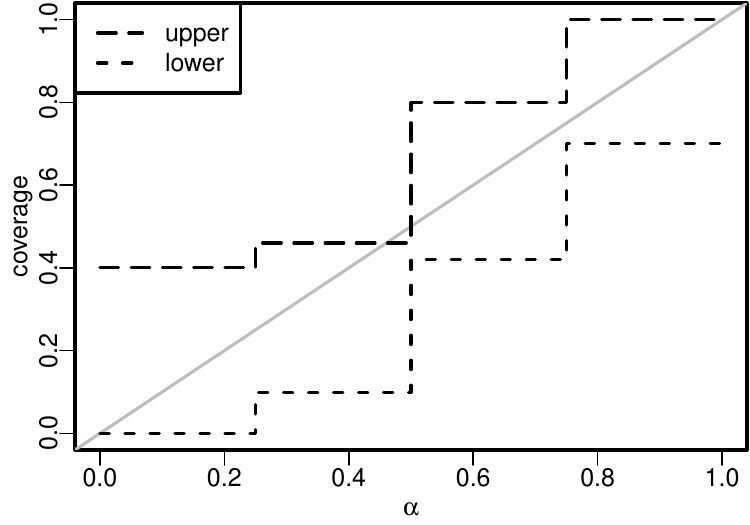}
\includegraphics[width=0.49\linewidth]{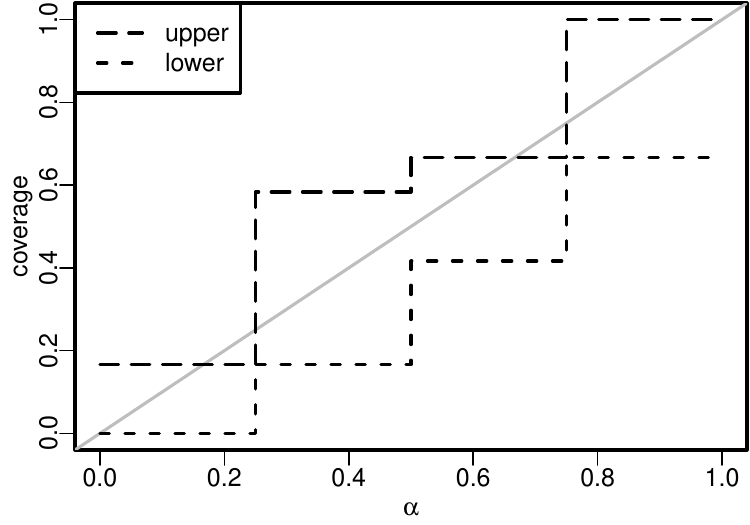}
\caption{Unconditional lower and upper quantile coverage in Example \ref{ex:TC+nUQC} (left) and Example \ref{ex:dnwPC} (right).  \label{fig:ucond_cov}}
\end{figure}

Finally, we verify the claims about the quantile-based notions $\QC$ and $\UQC$ in Table \ref{tab:app_examples}.  Essentially, the conditions in \eqref{eq:QC} and \eqref{eq:UQC} represent conditional and unconditional coverage, as discussed in detail by \citet{Gneiting2023b}.  Coverage can be computed for each forecast case $1, \ldots, N$ individually and can then be averaged conditionally on the forecast quantile, or unconditionally.  Figure \ref{fig:cond_cov} shows upper and lower coverage for each forecast case along with the corresponding quantile functions in Example \ref{ex:dPC}, which facilitate visual verification of quantile calibration: At each level $\alpha \in (0,1)$, we group forecasts according to the predicted $\alpha$-quantile and verify that within each group the average (probability-weighted) lower coverage at level $\alpha$ does not exceed $\alpha$, while the average upper coverage does not fall below $\alpha$.  Leveraging the piecewise constancy, it suffices to verify the $\QC$ constraints in \eqref{eq:QC} near the jump points of these functions.  Lastly, Figure \ref{fig:ucond_cov} shows unconditional lower and upper coverage in Examples \ref{ex:TC+nUQC} and \ref{ex:dnwPC}, where the $\UQC$ condition \eqref{eq:UQC} is notably violated.

\end{document}